\documentclass[10pt]{article}

\usepackage[margin=1in]{geometry}
\usepackage{times}
\usepackage{amsmath,amssymb,amsfonts}
\usepackage{mathrsfs}
\usepackage{graphicx}
\usepackage{booktabs}
\usepackage{multirow}
\usepackage{xcolor}
\usepackage{hyperref}
\usepackage{enumitem}
\usepackage{array}
\usepackage{algorithm}
\usepackage{lmodern}
\usepackage{algpseudocode}
\usepackage{tikz}
\usepackage{caption}
\usepackage{subcaption}
\usepackage{tikz}
\usetikzlibrary{arrows.meta,positioning,decorations.pathreplacing}

\usepackage{amsthm}
\newtheorem{theorem}{Theorem}
\newtheorem{lemma}{Lemma}
\newtheorem{proposition}{Proposition}
\newtheorem{corollary}{Corollary}
\newtheorem{definition}{Definition}
\newtheorem{remark}{Remark}

\newcommand{\keywords}[1]{%
  \par\noindent\textbf{Keywords:} #1
}
\DeclareMathOperator*{\argminA}{arg\,min} 
\newcommand{\paperTitle}{Curvature Cryptanalysis of Smooth Transformer Feed-Forward Networks}

\title{\vspace{-1em}
\paperTitle
}

\author{
Munawar Hasan$^{1,2}$, Apostol Vassilev$^{1}$ \vspace{0.2cm} \\
$^{1}$National Institute of Standards and Technology, USA \\
$^{2}$Michigan Technological University, USA \vspace{0.2cm} \\
\texttt{\{munawar.hasan, apostol.vassilev\}@nist.gov}\\ \texttt{munawarh@mtu.edu}
}

\date{}

\begin{document}
\maketitle

\begin{abstract}
    We show that smooth two-layer feed-forward networks (FFNs) expose an additional structural model-extraction channel under a chosen-input raw-output oracle at the FFN branch. We consider transformer FFN branches with GELU or SiLU activations under chosen-input raw-output access, without access to parameters, gradients, or internal activations. We characterize and exploit a second-order leakage channel in which projected input Hessians form different mixtures of the same hidden symmetric rank-one factors induced by the FFN input weights. Centered finite-difference queries expose these mixtures, and joint factorization recovers the hidden FFN direction dictionary up to sign and permutation. We formalize the resulting Hessian collection as a partially symmetric decomposition and establish conditions for local identifiability and stability. We exploit vector-output stencil reuse to reduce the structural query cost by a factor of 16.
    
    On independently trained CIFAR-10 vision transformers, only \(16\) projected Hessians, constructed offline from a single shared stencil of \(8{,}193\) black-box queries, recover the hidden FFN directions with average absolute cosine alignment above \(0.91\), with \(86.7\%\) of GELU and \(91.1\%\) of SiLU directions exceeding \(0.90\) alignment. Recovery remains high across independently trained models, repeated extraction runs, and all transformer blocks, and finite-difference queries closely match the structural recovery obtained with exact Hessians. The recovered structure also supports functional extraction. Keeping the recovered directions fixed and fitting only the remaining FFN parameters yields high-fidelity substitutes with more than \(92\%\) top-1 agreement, while their test accuracy remains within \(1.24\) and \(0.57\) percentage points of the GELU and SiLU targets, respectively. Finally, output rounding and Gaussian noise substantially reduce recovery under a fixed attack configuration, but adapting the finite-difference step restores average alignment to \(0.9188\) and \(0.9081\), respectively. These results establish an end-to-end path from black-box second-order observations to hidden FFN-structure recovery and functional replacement. Under the stated oracle model, smooth FFN curvature exposes internal parameter geometry that behavioral fidelity alone cannot reveal. Code is available on Github\footnote{\url{https://github.com/mhasan08/curvature-cryptanalysis}}
\end{abstract}


\section{Introduction}
\label{sec:introduction}

Modern neural networks are increasingly exposed through interfaces that return rich numerical outputs rather than only discrete predictions. This access has made model extraction a practical concern: by repeatedly querying a target model, an adversary may construct a substitute that closely reproduces its observable behavior~\cite{gharami2025clone, wang2026black}. Prior mathematical identification work shows that smooth-network derivatives reveal weights~\cite{10.1093/imaiai/iaaa036, FIEDLER2023123,FORNASIER2025101749}. Most extraction attacks, however, are evaluated behaviorally. They ask whether the substitute matches the target's predictions, probabilities, or logits. Such agreement can reveal that a function has been approximated, but it does not necessarily reveal whether the target model's hidden parameter structure has been recovered. Two networks can behave similarly on a query distribution while possessing very different internal representations and weights. This distinction becomes particularly important for transformers. Their feed-forward networks (FFNs) contain a large fraction of the model parameters and repeatedly transform the hidden representation at every block. If black-box access exposes not only the behavior of an FFN but also information about the geometry of its hidden weight vectors, then model extraction becomes a substantially more concerning problem: the interface leaks internal structure rather than merely supporting behavioral fidelity. This raises a natural question:

\begin{quote}
    \emph{Can black-box access reveal how a transformer FFN is internally parameterized---and can the leaked structure be turned into a high-fidelity functional extraction?}
\end{quote}

 We adopt the threat model in~\cite{carliniNN}, along with the assumptions on the oracle (black-box access to the target model) and the attacker's knowledge. Based on this, we demonstrate that smooth two-layer transformer FFN branches expose a practical structural extraction channel under a chosen-input raw-output oracle. Although the first-layer weight directions are never exposed directly, projected input Hessians contain them through a shared collection of symmetric rank-one factors. By varying 
 the output projection and, if necessary, the probe point, the adversary obtains different mixtures of the same hidden geometry, enabling recovery of the normalized first-layer directions. Crucially, this structural recovery is not only descriptive: once the directions are extracted, they can be held fixed while the remaining FFN parameters are fitted from additional black-box queries, yielding a high-fidelity functional replacement. We study an overcomplete vector-output setting, quantify query requirements, and connect structural recovery to downstream replacement.

This observation suggests a cryptanalytic view of model extraction. Rather than treating each query only as an input--output example for training a surrogate, we treat oracle responses as measurements from which hidden structure can be inferred. The resulting problem resembles recovery of a shared latent dictionary: estimate multiple projected Hessians, identify the rank-one factors common to them, and recover the corresponding first-layer directions up to the unavoidable sign and permutation ambiguities. Importantly, the structural question is distinct from functional imitation. A surrogate may achieve high agreement without identifying these directions, while a successful structural attack must recover them directly. Turning this observation into a practical black-box attack is not immediate. The adversary does not observe derivatives and must estimate second-order information using only oracle responses. The resulting Hessian collection is highly undercomplete relative to the number of hidden units, its factorization is nonconvex, finite-difference measurements amplify output perturbations at small step sizes, and low reconstruction error alone does not guarantee that the correct hidden factors have been identified. These issues require both a formal understanding of when the shared curvature representation is locally identifiable and an extraction procedure that can recover its factors from finite-difference measurements.

We address these challenges by developing a curvature-based structural extraction framework for smooth transformer FFNs with GELU~\cite{hendrycks2016gaussian} and SiLU~\cite{ramachandran2017searching,elfwing2018sigmoid} activations. We characterize the projected Hessians as a partially symmetric decomposition with shared rank-one atoms, derive conditions for local identifiability and stability, and instantiate the resulting attack using chosen-input finite differences followed by joint dictionary recovery. We then separately evaluate whether the extracted directions can support functional extraction when the remaining FFN parameters are fitted.

% Our experiments on independently trained vision transformers show that the leakage reveals internal structure rather than merely approximating behavior. Using only \(16\) projected Hessians, the attack recovers the hidden FFN directions with average absolute cosine alignment \(0.9644\) for GELU and \(0.9476\) for SiLU, with more than \(91\%\) of directions exceeding \(0.90\) alignment. Recovery remains high across independently trained models, repeated extraction runs, and transformer blocks. The recovered structure also supports functional extraction: fixing the recovered directions and fitting only the remaining FFN parameters produces substitutes with more than \(93\%\) top-1 agreement with the target classifiers. Finite-difference queries closely match the structural recovery obtained with exact Hessians, while output perturbations that appear protective at one measurement scale can be substantially weakened by adapting that scale. These results expose a qualitatively different form of model leakage. Black-box outputs can encode enough second-order information to reveal hidden FFN geometry and reuse that structure in a high-fidelity substitute. In this setting, the model interface leaks more than behavior: it leaks part of the parameterization itself.

Our experiments on independently trained vision transformers show that the leakage reveals internal structure rather than merely approximating behavior. Using only \(16\) projected Hessians, the attack recovers the hidden FFN directions with average absolute cosine alignment \(0.9187\) for GELU and \(0.9402\) for SiLU, with more than \(86\%\) of directions exceeding \(0.90\) alignment. Recovery remains high across independently trained models, repeated extraction runs, and transformer blocks. The recovered structure also supports functional extraction: fixing the recovered directions and fitting only the remaining FFN parameters produces substitutes with more than \(92\%\) top-1 agreement with the target classifiers. Finite-difference queries closely match the structural recovery obtained with exact Hessians, while output perturbations that appear protective at one measurement scale can be substantially weakened by adapting that scale. These results expose a qualitatively different form of model leakage. Black-box outputs can encode enough second-order information to reveal hidden FFN geometry and reuse that structure in a high-fidelity substitute. In this setting, the model interface leaks more than behavior: it leaks part of the parameterization itself.

\paragraph{Contributions.}
This work makes the following contributions:
\begin{itemize}
    \item We formalize black-box structural extraction of smooth transformer FFN branches under a chosen-input raw-output oracle. The adversary observes only \(g_\theta(x)=W_2\phi(W_1x+b_1)+b_2\) and receives no parameters, gradients, or internal activations. We distinguish recovery of the hidden FFN direction structure from the subsequent task of functional extraction.

    \item   We characterize and exploit second-order structural leakage in smooth transformer FFNs. Building on the known connection between smooth-network derivatives and hidden weights, we show how projected FFN Hessians expose a shared rank-one direction dictionary under a black-box security oracle. 

    \item We develop a theory explaining when the leaked curvature is sufficient for structural recovery. We formulate the Hessian collection as a partially symmetric decomposition, establish local identifiability of the hidden directions up to sign and permutation, derive a necessary measurement-count condition, and prove local stability under perturbed observations.

    \item We construct a practical black-box curvature-extraction attack. Projected Hessians are estimated using centered finite differences, requiring exactly \(2d^2+1\) oracle queries per distinct probe stencil, where \(d\) is the FFN input dimension. The shared direction dictionary is then recovered through random multi-restart joint optimization, with the final solution selected solely by the observed Hessian reconstruction loss. We optimize the attack to increase its power by taking into account  that the same finite-difference stencil at one probe point gives the Hessian of every output coordinate simultaneously, thus allowing a substantial reduction of the number of queries to the oracle.

    % \item We demonstrate that the leakage is reproducible, nontrivial, and operationally useful. Using only \(16\) projected Hessians, the attack recovers the hidden FFN directions with average absolute cosine alignment \(0.9644\) for GELU and \(0.9476\) for SiLU, remaining stable across independently trained target models, extraction runs, and transformer blocks. The recovered structure also supports functional extraction: fixing the recovered directions and fitting only the remaining FFN parameters yields substitutes with more than \(93\%\) top-1 agreement and accuracy drops of only \(0.90\) and \(0.62\) percentage points for GELU and SiLU, respectively. Finally, we show that output rounding and Gaussian noise can be substantially weakened by adapting the finite-difference step, indicating that fixed-step evaluations can overestimate extraction resistance.
    \item We demonstrate reproducible, end-to-end extraction. Only \(16\) projected Hessians recover GELU and SiLU directions with mean alignment \(0.9187\) and \(0.9402\), remaining stable across checkpoints, extraction runs, and transformer blocks. Fixing these directions yields substitutes with over \(92\%\) top-1 agreement and accuracy drops below \(1.25\) percentage points. Adaptive step selection also substantially weakens the evaluated rounding and noise defenses.
\end{itemize}

\paragraph{Paper organization.}
Section~\ref{sec:related-work} reviews related work. Section~\ref{sec:framework} defines the target FFN, the chosen-input raw-output oracle, and the structural and functional extraction objectives. Section~\ref{sec:curvature} develops the curvature leakage mechanism: it shows how projected FFN Hessians share hidden rank-one factors, reformulates these measurements as a partially symmetric decomposition, and establishes conditions for local identifiability and stability. It also clarifies the boundary of the analysis for residual connections, normalization, and piecewise-affine activations. Section~\ref{sec:attack} turns this structure into a practical black-box attack by estimating projected Hessians with finite differences, jointly recovering the shared first-layer direction dictionary, and optionally completing the remaining FFN parameters for full-model replacement. Section~\ref{sec:experiments} then evaluates structural recovery, functional replacement, query complexity, and robustness to perturbed oracle responses. Finally, section~\ref{sec:conclusion} presents conclusion.

\section{Related Work}
\label{sec:related-work}

Model extraction was initially studied as the problem of reproducing the functionality of a remotely deployed model through prediction queries. Tram\`er et al.~\cite{tramer2016stealing} demonstrated that prediction APIs can expose enough information to construct high-fidelity copies of several machine-learning model classes. Subsequent work extended this perspective to deep neural networks, where queried labels, probabilities, or logits are used as supervision for training a substitute model \cite{orekondy2019knockoff,jagielski2020high}. A separate line of work asks whether black-box access can reveal model parameters rather than merely approximate model behavior; for deep ReLU networks, such access can in principle recover architecture, weights, and biases by exploiting the boundaries between piecewise-linear regions~\cite{rolnick2020reverse}. Building on this direction, model extraction has been framed explicitly as a cryptanalytic problem, with differential chosen-input attacks recovering ReLU-network parameters to high precision~\cite{carliniNN}. Later work improved the computational efficiency of this approach and established polynomial-time extraction procedures for broad classes of ReLU networks~\cite{canales2024polynomial}. Theoretical work has likewise studied identifiability and query-based recovery of shallow and deep ReLU networks~\cite{bona2023parameter, liu2026navigating}. These approaches exploit a property fundamentally different from this paper. ReLU networks are piecewise affine, and their hidden units are revealed through activation boundaries or changes in local linear behavior. Smooth transformer activations such as GELU and SiLU have no corresponding piecewise-linear boundaries. We instead exploit their nonzero second-order curvature: projected input Hessians expose repeated mixtures of the same hidden symmetric rank-one factors. Our attack therefore replaces boundary localization with shared-curvature factorization and targets the first-layer direction dictionary of a smooth transformer FFN.

\paragraph{Derivative-based identification of smooth networks.} A separate line of mathematical work has shown that derivatives of smooth neural networks can reveal hidden weight structure. Fornasier et al.\cite{10.1093/imaiai/iaaa036} recover weights of shallow ridge-function networks using approximate first- and second-order derivatives, exploiting the rank-one matrices induced by hidden weight vectors. Subsequent work extends this approach to smooth deep and vector-valued networks through Hessian-derived “entangled weights” and establishes stable recovery from finite input-output samples~\cite{FIEDLER2023123}. More recent work~\cite{FORNASIER2025101749} addresses wide shallow networks with biases in the overcomplete regime $D<m<D^2$, recovering weight directions from second-order information before estimating the remaining parameters. Recently, Asselineau et al.~\cite{cryptoeprint:2026/253} extended model-stealing techniques to a broad class of non-linear
activation functions. Our work builds on this mathematical identification perspective but addresses a complementary security question: whether the same derivative structure constitutes an exploitable model-extraction channel in transformer FFN deployments. We study trained GELU/SiLU transformer FFNs under a chosen-input raw-output security oracle, formulate their projected Hessians as a shared partially symmetric decomposition, analyze local identifiability and conditioning in the resulting overcomplete vector-output problem, and connect recovered first-layer geometry to downstream functional replacement.

Model extraction of transformer-based systems has become increasingly popular, due to wide range of usability of transformer models. Algebraic and hybrid extraction attacks have been studied in grey-box language-model settings with a public encoder and private classification layer~\cite{zanella2021grey}. More recently, black-box extraction has been cast in cryptographic terms through bounded behavioral indistinguishability, where a surrogate is judged by whether it can be distinguished from the target under a bounded query procedure~\cite{hasan2026bounded}.  Several other approaches leverage observable properties of deployed models to infer internal or proprietary information~\cite{muvskardin2023testing,chaudhuri2025energon}.

Our setting is complementary but targets a different internal object. Rather than recovering a public-encoder classifier or an embedding/output projection, we study an internal smooth FFN branch and recover the normalized directions of its first linear map from second-order oracle measurements. We then retain those extracted directions during surrogate completion, linking parameter geometry recovery to high-fidelity functional extraction. To our knowledge, prior transformer extraction work has not exploited the shared Hessian structure of smooth FFNs to recover their hidden first-layer direction dictionary.

\paragraph{Output perturbation and extraction defenses.}
A broad family of model-extraction defenses attempts to reduce the information available through prediction interfaces, perturb returned outputs, redirect the optimization of a surrogate, or embed detectable watermarks~\cite{cheng2025misleader}. Such defenses are typically evaluated against learning-based behavioral extraction. Our robustness study addresses a narrower question specific to curvature extraction: whether output rounding or additive noise suppresses finite-difference Hessian recovery.

\section{Extraction Framework}
\label{sec:framework}

We begin by defining the target feed-forward network (FFN), the oracle interface, and the structural and functional extraction objectives. Let \(\mathcal{X}\subseteq\mathbb{R}^{d}\) denote the FFN input space, where \(d\) is the input dimension, and let \(k\) denote the output dimension. A probabilistic adversary $\mathcal{A}$ adaptively submits a query $x\in\mathcal{X}$ to an oracle for a designated FFN branch and receives the corresponding output vector in $\mathbb{R}^{k}$, while the model parameters, gradients, and intermediate activations remain hidden.

\subsection{Target FFN and Oracle Interface}
\label{sec:oracle}

Let $g_\theta:\mathbb{R}^{d}\rightarrow\mathbb{R}^{k}$ denote a two-layer FFN branch of the form $g_\theta(x) = W_2\phi(W_1x+b_1)+b_2$, where $W_1\in\mathbb{R}^{m\times d}$, $b_1\in\mathbb{R}^{m}$, $W_2\in\mathbb{R}^{k\times m}$, $b_2\in\mathbb{R}^{k}$. Here, \(m\) is the hidden width and \(\phi\) is a smooth activation function, such as GELU~\cite{hendrycks2016gaussian} or SiLU~\cite{ramachandran2017searching, elfwing2018sigmoid}. When \(k=d\), we define the residual map as: $F_\theta(x)=x+g_\theta(x)$. Since the queried input $x \in \mathcal{X}$ is known, $\mathcal{A}$ can recover the branch output as: $g_\theta(x)=F_\theta(x)-x$. Hence, direct branch-output access and residual-output access are equivalent under this oracle model. To model rounded oracle outputs, let \(Q_\Delta:\mathbb{R}^{k}\rightarrow\mathbb{R}^{k}\) round each coordinate of an output vector \(z\in\mathbb{R}^{k}\) to the nearest multiple of $\Delta>0$: $$Q_\Delta(z)_r = \Delta\operatorname{round}\!\left(\frac{z_r}{\Delta}\right), \qquad r=1,\ldots,k.$$ For example, rounding to $p$ decimal places corresponds to $\Delta=10^{-p}$. Since each coordinate changes by at most half the quantization interval: $\|Q_\Delta(z)-z\|_\infty \leq \frac{\Delta}{2}$. We define \(Q_0(z)=z\) to represent an unrounded output.

\begin{definition}[Chosen-input raw-output FFN oracle]
\label{def:oracle}
Let $g_\theta:\mathcal{X}\subseteq\mathbb{R}^{d} \rightarrow \mathbb{R}^{k}$ be the target FFN. For a quantization interval \(\Delta\geq 0\) and a perturbation bound \(\tau\geq 0\), then for each adaptively selected query \(x\in\mathcal{X}\), the oracle \(\mathcal{O}_{g}^{(\Delta,\tau)}\) returns:
\begin{equation}
\mathcal{O}_{g}^{(\Delta,\tau)}(x)
=
Q_\Delta\!\left(g_\theta(x)+\xi_x\right),
\qquad
\xi_x\in\mathbb{R}^{k},
\qquad
\|\xi_x\|_\infty\leq\tau.
\label{eq:noisy-oracle}
\end{equation}
Here, \(\xi_x\) denotes an additive perturbation applied to the output associated with query \(x\). We define ideal raw-output oracle as: $ \mathcal{O}_{g} = \mathcal{O}_{g}^{(0,0)}$. $\mathcal{A}$ may choose each query as a function of previous oracle responses but receives neither parameters, gradients, nor internal activations. One returned vector in \(\mathbb{R}^{k}\) counts as one oracle query.
\end{definition}

\paragraph{Query and access model.}
The query \(x\in\mathcal{X}\) is an FFN input representation rather than necessarily a raw image or token sequence. The set \(\mathcal{X}\) may be the full representation space, a bounded subset, or a neighborhood of naturally occurring representations. The oracle therefore provides block-level, representation-space access, rather than a conventional label-only or final-logit interface. Such access may arise when an FFN module can be invoked independently, as in modular or split-inference deployments.

\paragraph{Adversary knowledge.}
The adversary knows the FFN architecture, including \(d\), \(k\), \(m\), the activation family \(\phi\), the residual wiring, and the target block. The parameters $W_1$, $b_1$, $W_2$ and $b_2$ remain unknown. Inferring the hidden width, activation family, or target-block location is outside the scope of this work.

\subsection{Extraction Objectives}
\label{sec:objectives}

Under the oracle model defined above, we now distinguish structural extraction from functional extraction. Structural extraction seeks to recover the directions of the rows of the first-layer matrix \(W_1\). Functional extraction instead seeks a surrogate branch whose outputs approximate those of the target FFN. When inserted into the original classifier, such a surrogate may also preserve the behavior of the complete model.

\begin{definition}[\((\varepsilon,\delta)\)-functional extraction][Adapted from~\cite{carliniNN}]
\label{def:functional-extraction}
Let $f,\widehat f:\mathcal{X}\rightarrow\mathbb{R}^{k}$. Let \(\mathcal{D}\) be a distribution over \(\mathcal{X}\), and let $\ell: \mathbb{R}^{k}\times\mathbb{R}^{k} \rightarrow\mathbb{R}_{\geq 0}$ be a loss function measuring the discrepancy between the outputs of \(f\) and \(\widehat f\). Then, we say that \(\widehat f\) is an \((\varepsilon,\delta)\)-functional extraction of \(f\) over \(\mathcal{D}\), with respect to \(\ell\), if:
\begin{equation}
    \Pr_{x\sim\mathcal{D}} \left[\ell\!\left(f(x),\widehat f(x)\right) \leq\varepsilon \right] \geq 1-\delta.
\end{equation}
\end{definition}

For system-level evaluation, let \(M\) denote the target classifier and let \(\widehat M\) denote the same classifier with the target FFN replaced by a surrogate branch. Their predictive agreement over an image distribution \(\mathcal{D}_{\mathrm{img}}\) is given by:
\begin{equation}
\label{eq:system-agreement}
    \operatorname{Agree}(M,\widehat M;\mathcal{D}_{\mathrm{img}}) = \Pr_{z\sim\mathcal{D}_{\mathrm{img}}} \left[ \arg\max_i M_i(z) = \arg\max_i\widehat M_i(z) \right].
\end{equation}
We additionally evaluate target and surrogate accuracy, KL divergence, and centered-logit mean-squared error. These metrics quantify empirical functional fidelity without requiring the selection of a single \((\varepsilon,\delta)\) pair.

Let $w_j^\top = W_{1,j:}$ denote the \(j\)-th row of \(W_1\), where $w_j\in\mathbb{R}^{d}$. We define its unit-norm direction as $u_j = \frac{w_j}{\|w_j\|_2}$. Then the structural extraction objective is to recover the collection: $U=\{u_1,\ldots,u_m\}$. Only the directions are targeted at this stage; their magnitudes are handled later during surrogate completion. Since the projected Hessians contain each direction through \(u_ju_j^\top\), the vectors \(u_j\) and \(-u_j\) produce the same curvature contribution. Furthermore, hidden-unit indices are not intrinsically observable. Structural recovery is therefore defined only up to sign and permutation.

\begin{definition}[Permutation-sign direction recovery]
\label{def:direction-recovery}
Let $U=\{u_j\}_{j=1}^{m}$ and $\widehat U=\{\widehat u_j\}_{j=1}^{m}$ be collections of unit vectors. Their matched direction-recovery score is defined as
\begin{equation}
\label{eq:dirrec}
    \operatorname{DirRec}(\widehat U,U) = \max_{\pi\in S_m} \frac{1}{m} \sum_{j=1}^{m} \left| \left\langle \widehat u_{\pi(j)},u_j \right\rangle \right|.
\end{equation}
Let \(\pi^\star\) denote an optimal permutation. For a threshold \(\rho\in[0,1]\), then we define:
\begin{equation}
\label{eq:hit-rate}
    \operatorname{Hit}_{\rho}(\widehat U,U) = \frac{1}{m} \sum_{j=1}^{m} \mathbf{1} \left\{ \left| \left\langle \widehat u_{\pi^\star(j)},u_j \right\rangle \right| \geq\rho \right\}.
\end{equation}
\end{definition}

The metric \(\operatorname{DirRec}\) reports the average matched alignment, while \(\operatorname{Hit}_{0.90}\) reports the fraction of directions whose matched absolute cosine similarity is at least \(0.90\). Both metrics require the true first-layer directions and are used only for scientific evaluation; they are not available to \(\mathcal{A}\) during extraction.

\section{Curvature Leakage and Identifiability}
\label{sec:curvature}
This section establishes the structural leakage mechanism underlying the attack. For a smooth FFN, scalar projections of the output produce input Hessians that are different linear combinations of the same hidden rank-one matrices induced by the rows of \(W_1\). Consequently, repeated curvature measurements expose a shared dictionary of first-layer directions. We formalize this structure, characterize its local identifiability, and analyze its stability under measurement perturbations.

\subsection{Residual Cancellation}
\label{sec:residual-cancellation}
We first show that the residual identity does not contribute to the second-order signal used by the extraction method.

\begin{lemma}[Residual identity has no curvature]
\label{lem:residual-cancellation}
Let \(k=d\), and $F_\theta(x)=x+g_\theta(x)$ be the residual map associated with the FFN branch \(g_\theta\) (cf., Section~\ref{sec:oracle}). Then, for any \(c\in\mathbb{R}^{d}\), we define: $$\nabla_x^2\!\left(c^\top F_\theta(x)\right) = \nabla_x^2\!\left(c^\top g_\theta(x)\right).$$
\end{lemma}

\begin{proof}
Since $c^\top F_\theta(x) = c^\top x+c^\top g_\theta(x)$, and \(c^\top x\) is linear in \(x\), its Hessian is zero. Hence, we have:
\begin{equation}
    \nonumber
    \begin{split}
        \nabla_x^2\!\left(c^\top F_\theta(x)\right) &= \nabla_x^2(c^\top x) + \nabla_x^2\!\left(c^\top g_\theta(x)\right) \\
        &= \nabla_x^2\!\left(c^\top g_\theta(x)\right).
    \end{split}
\end{equation}
\end{proof}
Thus, observing the residual output preserves exactly the curvature of the nonlinear FFN branch, i.e., the identity path neither adds to nor obscures this second-order signal.

\begin{remark}[Effect of normalization before the FFN]
\label{rem:layernorm}
Lemma~\ref{lem:residual-cancellation} applies to the affine residual identity, but not to nonlinear transformations applied before the FFN. Consider $F_\theta(x) = x+g_\theta(\operatorname{LN}(x))$, where $\operatorname{LN}$ denotes layer normalization. Let $z=\operatorname{LN}(x)$ and $q(z)=c^\top g_\theta(z)$. Then the multivariate chain rule gives following:
\begin{equation}
    \nonumber
    \nabla_x^2 q(\operatorname{LN}(x)) = J_{\operatorname{LN}}(x)^\top \nabla_z^2 q(z)J_{\operatorname{LN}}(x) + \sum_{r=1}^{d}\frac{\partial q}{\partial z_r} \nabla_x^2\operatorname{LN}_r(x).
\end{equation}
The second term contains curvature introduced by LayerNorm itself. Our primary analysis therefore treats the post-normalization representation \(z\) as the FFN query variable. An oracle placed before LayerNorm requires separating these additional normalization-curvature terms from the curvature of the FFN branch.
\end{remark}

\subsection{Shared Hessian Mixtures}
\label{sec:hessian-mixture}
We next show that scalar projections of the FFN output produce Hessians that share the same first-layer rank-one factors. Recall that \(w_j\in\mathbb{R}^{d}\) denotes the \(j\)-th first-layer weight vector and \(v_j\in\mathbb{R}^{k}\) denotes the corresponding \(j\)-th column of \(W_2\). For any projection vector \(c\in\mathbb{R}^{k}\), define the scalar response $s_c(x)=c^\top g_\theta(x)$.

\begin{lemma}[Hessian mixture of a smooth FFN]
\label{lem:hessian-mixture}
Let $\phi$ be twice differentiable. Then we show that: 
\begin{equation}
    \label{eq:hessian-mixture}
    \nabla_x^2 s_c(x) = \sum_{j=1}^{m} (c^\top v_j) \phi''(w_j^\top x+b_{1,j}) w_jw_j^\top.
\end{equation}
\end{lemma}

\begin{proof}
Using the row and column notation above, the projected response can be written as:
$$ s_c(x) = \sum_{j=1}^{m} (c^\top v_j) \phi(w_j^\top x+b_{1,j}) + c^\top b_2.$$
Differentiating with respect to \(x\) gives: $\nabla_x s_c(x) = \sum_{j=1}^{m} (c^\top v_j) \phi'(w_j^\top x+b_{1,j})w_j$. Differentiating once more yields Equation~\eqref{eq:hessian-mixture}.
\end{proof}

Now, for a collection of query--projection pairs $\{(x_t,c_t)\}_{t=1}^{T}$, define $H_t = \nabla_x^2 s_{c_t}(x_t)$ and $ \alpha_{t,j} = (c_t^\top v_j)
\phi''(w_j^\top x_t+b_{1,j})$. Then
\begin{equation}
    \label{eq:shared-hessian-dictionary}
    H_t = \sum_{j=1}^{m} \alpha_{t,j}w_jw_j^\top.
\end{equation}
The coefficients \(\alpha_{t,j}\) vary with the query point \(x_t\) and the projection \(c_t\), whereas the rank-one matrices \(w_jw_j^\top\) are shared across all measurements.

\subsection{Partially Symmetric Curvature Decomposition}
\label{sec:tensor-decomposition}
The shared Hessian mixtures in Equation~\eqref{eq:shared-hessian-dictionary} can be organized across all \(T\) measurements as a partially symmetric tensor decomposition. Define the normalized first-layer directions $ u_j = \frac{w_j}{\|w_j\|_2}$, where $j=1,\ldots,m$, and absorb the corresponding row magnitudes into the measurement
coefficients: $\gamma_{t,j} = \alpha_{t,j}\|w_j\|_2^2$. Then
\begin{equation}
    \label{eq:normalized-hessian-mixture}
    H_t = \sum_{j=1}^{m} \gamma_{t,j}u_ju_j^\top.
\end{equation}

\begin{proposition}[Partially symmetric curvature decomposition]
\label{prop:tensor-decomposition}
Let $\mathscr{H}\in\mathbb{R}^{d\times d\times T}$  be the third-order tensor whose \(t\)-th frontal slice is \(H_t\). For each hidden unit \(j\), define its coefficient vector $\gamma_j = (\gamma_{1,j},\ldots,\gamma_{T,j})^\top \in\mathbb{R}^{T}$. Then
\begin{equation}
\mathscr{H}
=
\sum_{j=1}^{m}
u_j\otimes u_j\otimes\gamma_j.
\label{eq:partial-symmetric-decomp}
\end{equation}
\end{proposition}

\begin{proof}
For each \(t=1,\ldots,T\), the \(t\)-th frontal slice of the right-hand side of Equation~\eqref{eq:partial-symmetric-decomp} is $\sum_{j=1}^{m} \gamma_{t,j}u_ju_j^\top = H_t$. Therefore, the tensor equality holds.
\end{proof}

Equation~\eqref{eq:partial-symmetric-decomp} expresses the curvature measurements as a shared rank-one decomposition: the first two modes contain the same normalized direction \(u_j\), while the third mode records how the coefficient of that direction varies across measurements. Structural extraction can therefore be viewed as recovering the direction factors \(\{u_j\}_{j=1}^{m}\) from this partially symmetric decomposition. This formulation is more informative than an unconstrained subspace model. Even when \(T<m\), each matrix atom is restricted to the symmetric rank-one form \(u_ju_j^\top\), rather than being an arbitrary element of the span of the observed Hessians.

\subsection{Conditional Local Identifiability}
\label{sec:local-identifiability}
The partially symmetric decomposition in Equation~\eqref{eq:partial-symmetric-decomp} shows how a collection of directions and measurement coefficients generates the observed Hessians. We now ask the inverse question: \emph{under what conditions do the Hessians determine these factors uniquely in a neighborhood of the true solution?}

Let $ U=(u_1,\ldots,u_m)$ and $u_j\in\mathbb{S}^{d-1}$, where $\mathbb{S}^{d-1} = \{u\in\mathbb{R}^{d}:\|u\|_2=1\}$. Let $\Gamma=[\gamma_{t,j}] \in\mathbb{R}^{T\times m}$. For each measurement \(t\), define $$\Psi_t(U,\Gamma) = \sum_{j=1}^{m} \gamma_{t,j}u_ju_j^\top,$$
and collect all \(T\) measurements through the map: $ \Psi(U,\Gamma) = \bigl( \Psi_1(U,\Gamma), \ldots, \Psi_T(U,\Gamma) \bigr)$. Hence, \(\Psi\) maps a normalized direction dictionary and its measurement-specific coefficients to the corresponding collection of symmetric Hessian matrices. By construction in Equation~\eqref{eq:normalized-hessian-mixture}, the normalized directions satisfy $u_j=\frac{w_j}{\|w_j\|_2}$, such that $\|u_j\|_2=1$, and hence \(u_j\in\mathbb{S}^{d-1}\). Consequently, an infinitesimal perturbation \(\dot u_j\) that preserves this normalization must satisfy $u_j^\top \dot u_j=0$. We denote by \(J_{\mathrm{res}}(U,\Gamma)\) the differential of \(\Psi\) restricted to these tangent perturbations, together with arbitrary coefficient perturbations \(\dot\Gamma\).

\begin{theorem}[Local identifiability criterion]
\label{thm:local-identifiability}
Let \(J_{\mathrm{res}}(U,\Gamma)\) denote the restricted differential of \(\Psi\) over perturbations satisfying $u_j^\top \dot u_j=0$, for $j=1,\ldots,m$. For $t=1,\ldots,T$, if the only perturbations satisfying
\begin{equation}
\label{eq:jacobian-kernel}
\sum_{j=1}^{m} \left[ \dot\gamma_{t,j}u_ju_j^\top + \gamma_{t,j} \left( \dot u_ju_j^\top + u_j\dot u_j^\top \right) \right] = 0,
\end{equation}
are $\dot U=0$ and $\dot\Gamma=0$, then \(J_{\mathrm{res}}(U,\Gamma)\) is injective, and the factorization $ H_t = \sum_{j=1}^{m} \gamma_{t,j}u_ju_j^\top$ for $t=1,\ldots,T$, is locally identifiable up to simultaneous permutation of the atom--coefficient pairs and sign changes \(u_j\mapsto -u_j\).
\end{theorem}

\begin{proof}[Proof sketch]
Equation~\eqref{eq:jacobian-kernel} describes infinitesimal perturbations of \(U\) and \(\Gamma\) that produce no first-order change in any of the observed Hessians. These perturbations form the kernel of the restricted differential \(J_{\mathrm{res}}(U,\Gamma)\). By assumption, this kernel is trivial, so \(J_{\mathrm{res}}(U,\Gamma)\) is injective, equivalently full column rank, at the true factorization. Since the differential varies continuously with \(U\) and \(\Gamma\), full column rank is preserved in a sufficiently small neighborhood of \((U,\Gamma)\). The constant-rank theorem therefore implies that \(\Psi\) is locally one-to-one onto its image. Hence, sufficiently nearby factorizations that produce the same Hessian collection must coincide, apart from the remaining discrete symmetries. The unit-norm constraints remove the continuous scaling ambiguity between \(u_j\) and \(\gamma_{t,j}\), while sign changes and simultaneous permutations of the atom--coefficient pairs leave every Hessian unchanged.
\end{proof}

The theorem is deliberately conditional: it does not assert that every FFN or every choice of probe--projection pairs produces an injective restricted Jacobian. Instead, it isolates the local nondegeneracy condition under which the shared Hessian measurements uniquely determine the normalized directions and coefficients, apart from their unavoidable discrete symmetries.

\begin{corollary}[Necessary dimension count for the relaxed factorization]
\label{cor:dimension-count}
If \(\Gamma\) is treated as a free coefficient matrix, a necessary condition for the differential in Theorem~\ref{thm:local-identifiability} to be injective is: $m(d-1)+Tm \leq T\frac{d(d+1)}{2}$. When $\frac{d(d+1)}{2}>m$, this condition is equivalent to
\begin{equation}
    \label{eq:T-lower-bound}
    T \geq \frac{m(d-1)} {\frac{d(d+1)}{2}-m}.
\end{equation}
This bound is necessary but not sufficient.
\end{corollary}

\begin{proof}
The \(m\) normalized directions contribute \(m(d-1)\) continuous degrees of freedom, since each \(u_j\) lies on \(\mathbb{S}^{d-1}\). The free coefficient matrix contributes \(Tm\) additional degrees of freedom. On the observation side, each symmetric \(d\times d\) Hessian contains at most \(d(d+1)/2\) independent scalar entries, so \(T\) Hessians contain at most $T \cdot \frac{d(d+1)}{2}$ independent observations. An injective differential cannot map a higher-dimensional tangent space into a lower-dimensional observation space.
\end{proof}

For the dimensions used in our evaluation, $d=64$, $m=128$ Equation~\eqref{eq:T-lower-bound} gives $ T \geq \left\lceil \frac{128(64-1)} {\frac{64(65)}{2}-128} \right\rceil = 5$. This calculation does not establish that \(T=5\) is sufficient. It only shows that the structured factorization is not dimensionally constrained to require \(T\geq m\). The bound is also conservative relative to the original FFN model because it treats the coefficients in \(\Gamma\) as unconstrained nuisance parameters.

\paragraph{Numerical verification of the identifiability condition.}
Theorem~\ref{thm:local-identifiability} reduces local identifiability to injectivity of the restricted differential \(J_{\mathrm{res}}(U,\Gamma)\). We evaluate this condition numerically at the trained FFN instances used in our experiments. For each trained model and for each checkpoint, we fix one probe $x_p$ and use a nested sequence of output projections $c_t$. We then construct the FFN-induced coefficients $\gamma_{t,j} = (c_t^\top v_j) \phi''(w_j^\top x_p+b_{1,j}) \|w_j\|_2^2$, and evaluate \(J_{\mathrm{res}}(U,\Gamma)\) at the resulting \((U,\Gamma)\). We define $\kappa = \sigma_{\min} \!\left( J_{\mathrm{res}}(U,\Gamma) \right)$ and additionally report the scale-normalized quantity: $\frac{\kappa}{\sigma_{\max}(J_{\mathrm{res}})}$. A smallest singular value that is separated from zero supports the injectivity hypothesis of Theorem~\ref{thm:local-identifiability} up to the stated numerical tolerance. The magnitude of \(\kappa\), and particularly its scale-normalized counterpart, quantifies the local conditioning of the factorization. Section~\ref{sec:empirical-identifiability} evaluates these quantities across the independently trained GELU and SiLU models used in our experiments.

\subsection{Local Stability}
\label{sec:local-stability}
Local identifiability establishes uniqueness in a neighborhood of the true factorization, but does not by itself guarantee robustness to measurement error. We therefore examine how perturbations in the observed Hessians affect the recovered directions and coefficients. To account for the unavoidable sign and permutation ambiguities, define the symmetry-aware factor distance by:
\begin{equation}
    \label{eq:factor-distance}
    \operatorname{dist}_{\pm,\pi}^{\,2} \left((\widehat U,\widehat\Gamma),(U,\Gamma) \right)  = \min_{\substack{\pi\in S_m\\ \sigma\in\{\pm1\}^{m}}} \left[ \sum_{j=1}^{m} \left\| \widehat u_{\pi(j)}-\sigma_j u_j \right\|_2^2 + \sum_{j=1}^{m} \left\| \widehat\Gamma_{\cdot,\pi(j)} - \Gamma_{\cdot,j} \right\|_2^2 \right].
\end{equation}
The permutation \(\pi\) acts simultaneously on the recovered directions and their coefficient columns. The signs do not act on \(\Gamma\), since, $(-u_j)(-u_j)^\top = u_ju_j^\top$.

\begin{proposition}[Local stability under Hessian perturbations]
\label{prop:local-stability}
Assume the local identifiability condition of Theorem~\ref{thm:local-identifiability}. Let \(\kappa>0\) denote the smallest singular value of the differential of \(\Psi\), restricted to the tangent space of the normalized factorization. Suppose that the observed Hessians satisfy $\widehat H_t = H_t+E_t$, where $t=1,\ldots,T$, and define the aggregate perturbation magnitude: $$\eta = \left(\sum_{t=1}^{T} \|E_t\|_F^2 \right)^{1/2}.$$
For sufficiently small \(\eta\), let \((\widehat U,\widehat\Gamma)\) be a local least-squares solution within the identifiable neighborhood whose stacked reconstruction residual is \(O(\eta)\). Then: $\operatorname{dist}_{\pm,\pi}\left((\widehat U,\widehat\Gamma),(U,\Gamma) \right) = O\!\left(\frac{\eta}{\kappa}\right)$. Consequently, $$\operatorname{DirRec}(\widehat U,U) \geq 1-O\!\left(\frac{\eta^2}{\kappa^2}\right).$$
\end{proposition}

\begin{proof}[Proof sketch]
Within the identifiable neighborhood, the restricted differential of \(\Psi\) is injective and has smallest singular value \(\kappa\). The local inverse is therefore Lipschitz on the image, with sensitivity proportional to \(1/\kappa\). A Hessian perturbation of stacked norm \(\eta\), together with a reconstruction residual of the same order, produces factor error \(O(\eta/\kappa)\), modulo the sign and permutation symmetries. After optimal alignment, the recovered and true directions are unit vectors. For unit vectors \(u\) and \(\widehat u\), we have: $1-\langle \widehat u,u\rangle = \frac{1}{2}\|\widehat u-u\|_2^2$. Thus, direction error of order \(O(\eta/\kappa)\) produces a cosine-alignment loss of order \(O(\eta^2/\kappa^2)\).
\end{proof}

The quantity \(\kappa\) separates uniqueness from numerical conditioning. A factorization may be locally identifiable while remaining highly sensitive to measurement error when \(\kappa\) is small.

\subsection{Piecewise-Affine Boundary Case}
\label{sec:relu-boundary}
The preceding curvature analysis relies on a twice-differentiable activation with nonzero second derivative. Piecewise-affine activations form an important boundary case because their classical curvature vanishes away from activation boundaries.

\begin{corollary}[Vanishing curvature for piecewise-affine activations]
\label{cor:relu}
Suppose that \(\phi\) is piecewise affine. At every point where its classical second derivative exists, $\phi''(z)=0$. Consequently, away from activation boundaries, $\nabla_x^2\left(c^\top g_\theta(x)\right) =0$. In particular, the smooth-curvature leakage channel vanishes almost everywhere for a ReLU FFN.
\end{corollary}

\begin{proof}
Within any region having a fixed activation pattern, a piecewise-affine FFN is an affine function of its input. Every scalar projection of its output is therefore affine within that region and has zero classical Hessian.
\end{proof}

\begin{remark}[Finite differences across activation boundaries]
\label{rem:relu-finite-difference}
A finite-difference stencil may cross a ReLU activation boundary and produce a nonzero second difference even when the classical Hessian at the query point is zero. Such a response is caused by a nonsmooth change in activation pattern rather than by the smooth Hessian mixture of Lemma~\ref{lem:hessian-mixture}. Recovering ReLU networks through boundary localization therefore constitutes a different extraction mechanism and is outside the scope of the present smooth-curvature analysis.
\end{remark}

\begin{figure}[t]
\centering
\begin{tikzpicture}[
    node distance=1.6cm and 0.9cm,
    >=Latex,
    every node/.style={align=center},
    box/.style={
        draw,
        rounded corners,
        thick,
        minimum height=1.2cm,
        minimum width=3.7cm,
        inner sep=6pt
    }
]

\node[box] (hess) {Projected\\ Hessian Collection};

\node[box, right=of hess] (dict) {Shared Rank-One\\ Dictionary Recovery};

\node[box, right=of dict] (comp) {Optional Surrogate\\ Completion};

\draw[->, thick] (hess) -- (dict);
\draw[->, thick] (dict) -- (comp);

% Bottom brace for structural extraction
\draw[decorate, decoration={brace, mirror, amplitude=6pt}, thick]
([yshift=-0.55cm]hess.south west) -- ([yshift=-0.55cm]dict.south east)
node[midway, below=8pt] {\small Structural Extraction};

% Bottom brace for functional evaluation only
\draw[decorate, decoration={brace, mirror, amplitude=6pt}, thick]
([yshift=-0.55cm]comp.south west) -- ([yshift=-0.55cm]comp.south east)
node[midway, below=8pt] {\small Functional Extraction};

\end{tikzpicture}
\caption{Pipeline of the curvature-based extraction attack. The first two stages---projected Hessian collection and shared dictionary recovery---form the structural extraction procedure. The final stage is used only to evaluate the functional value of the recovered directions.}
\label{fig:attack-pipeline}
\end{figure}
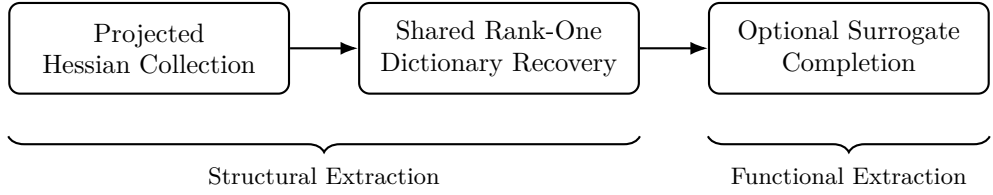

\section{Curvature-Based Extraction Attack}
\label{sec:attack}

We now turn the shared-curvature structure established in Section~\ref{sec:curvature} into a black-box extraction procedure. The attack has two principal stages. It first estimates a collection of projected input Hessians from chosen-input oracle responses, and then jointly factorizes those Hessians to recover the shared first-layer direction set. An optional completion stage fits the remaining branch parameters while keeping the recovered directions fixed. The resulting pipeline is shown in Figure~\ref{fig:attack-pipeline}. The structural extraction claim concerns the first two stages. The completion stage is used for the functional extraction.

\subsection{Projected Hessian Collection}
\label{sec:projected-hessian-collection}
The first stage of the attack constructs multiple Hessian measurements whose coefficients vary across probes but whose normalized first-layer direction atoms remain shared.

\paragraph{Probe and projection selection.}
For each measurement index \(t\in\{1,\ldots,T\}\), the adversary selects a probe point $x_t\in\mathcal{X}\subseteq\mathbb{R}^{d}$ and an output projection $c_t\in\mathbb{R}^{k}$. We sample the probes and projections as: $x_t\sim\mathcal{N}(0,\sigma_x^2I_d)$ and $c_t\sim\mathcal{N}(0,I_k)$ respectively, where \(\sigma_x>0\) controls the probe scale. Samples are chosen so that all points required by the finite-difference stencil lie within the oracle query domain \(\mathcal{X}\). Now, for each pair \((x_t,c_t)\), define the scalar projected response as: $s_t(x) = c_t^\top g_\theta(x)$ and its input Hessian at \(x_t\): $H_t = \nabla_x^2s_t(x_t)$. Since the oracle returns the complete vector \(g_\theta(x)\), the scalar projection \(c_t^\top g_\theta(x)\) is computed locally and requires no additional oracle query. From Equation~\eqref{eq:normalized-hessian-mixture}, we know that each projected Hessian has the form $ H_t = \sum_{j=1}^{m} \gamma_{t,j}u_ju_j^\top$. Changing \(x_t\) or \(c_t\) changes the coefficient vector \(\gamma_{t,\cdot}\), while the direction atoms \(u_ju_j^\top\) remain shared across measurements. Multiple probe--projection pairs therefore produce different linear combinations of the same unknown structural dictionary.

\paragraph{Finite-difference estimation.}
Under the perturbed oracle (refer Definition~\ref{def:oracle}), we define the observed scalar response: $\widetilde s_t(x) = c_t^\top\mathcal{O}_{g}^{(\Delta,\tau)}(x)$. For a finite-difference step \(h>0\), the diagonal Hessian entries are estimated by:
\begin{equation}
    \label{eq:fd-diagonal}
    \widehat H_{t,ii} = \frac{ \widetilde s_t(x_t+he_i) - 2\widetilde s_t(x_t) + \widetilde s_t(x_t-he_i)}{h^2},
\end{equation}
where \(e_i\) is the \(i\)-th standard basis vector. For \(i\neq j\), the mixed entries are estimated by:
\begin{equation}
\label{eq:fd-offdiagonal}
    \widehat H_{t,ij} = \frac{\widetilde s_t(x_t+he_i+he_j) - \widetilde s_t(x_t+he_i-he_j) - \widetilde s_t(x_t-he_i+he_j) + \widetilde s_t(x_t-he_i-he_j)}{4h^2}.
\end{equation}
Symmetry is imposed by setting: $\widehat H_{t,ji} = \widehat H_{t,ij}$. The centered stencils cancel the constant and first-order Taylor terms. If
\(s_t\) is four times continuously differentiable throughout the stencil neighborhood, the resulting truncation error is \(O(h^2)\).

\paragraph{Perturbation amplification.}
For each oracle call, let $\varepsilon_t(x) = \widetilde s_t(x)-s_t(x)$  denote the induced scalar response error. Then using Definition~\ref{def:oracle} we have,
\begin{equation}
    \label{eq:scalar-oracle-error}
    |\varepsilon_t(x)| \leq \|c_t\|_1 \left(\tau+\frac{\Delta}{2} \right) =: \tau_{c_t}.
\end{equation}

\begin{proposition}[Finite-difference error amplification]
\label{prop:finite-difference-error}
Let \(M_4\) denote a uniform bound on the relevant fourth-order partial derivatives of \(s_t\) over the finite-difference stencil neighborhood, such that $\left|\partial_a\partial_b\partial_c\partial_d s_t(x) \right| \le M_4$, for all relevant index tuples \((a,b,c,d)\) and all \(x\) in the stencil neighborhood. Then each estimated Hessian entry satisfies:
$$
\left| \widehat H_{t,ij}-H_{t,ij} \right| \leq C_{ij}^{\mathrm{tr}}M_4h^2 + C_{ij}^{\mathrm{pert}} \frac{\tau_{c_t}}{h^2},
$$
where \(C_{ij}^{\mathrm{tr}}>0\) is a stencil-dependent truncation constant. For the perturbation term $C_{ii}^{\mathrm{pert}}=4$ for diagonal entries, while $C_{ij}^{\mathrm{pert}}=1$, where $i\neq j$, for mixed entries.
\end{proposition}

\begin{proof}[Proof sketch]
A Taylor expansion through fourth order shows that the centered finite-difference stencils have truncation error of order \(M_4h^2\) after division by \(h^2\). For a diagonal entry, the numerator coefficients of the stencil are \(1,-2,1\). Their absolute values sum to \(4\), so perturbing each scalar response by at most \(\tau_{c_t}\) contributes at most $\frac{4\tau_{c_t}}{h^2}$. For a mixed entry, the four scalar responses have coefficients of magnitude \(1/(4h^2)\). The sum of their absolute values is therefore \(1/h^2\), giving a perturbation contribution bounded by $\frac{\tau_{c_t}}{h^2}$. Combining the truncation and perturbation terms gives the stated bound.
\end{proof}

The resulting entry wise error tradeoff is given by:
\begin{equation}
    \label{eq:fd-error-tradeoff}
    O(M_4h^2) + O\!\left(\frac{\tau_{c_t}}{h^2} \right).
\end{equation}
Thus, decreasing \(h\) reduces truncation error but amplifies oracle perturbations, whereas increasing \(h\) suppresses perturbation amplification at the cost of larger truncation error. This tradeoff motivates the adaptive finite-difference experiments in Section~\ref{sec:adaptive-robustness}.

\paragraph{Query complexity - a schoolbook example.}
A complete Hessian estimate uses one query at \(x_t\), two queries for each diagonal entry, and four queries for each unordered off-diagonal pair. Hence, we have:
\begin{equation}
    \label{eq:queries-per-hessian}
    \begin{split}
    Q_{\mathrm{Hess}}(d) &= 1+2d+4\binom{d}{2}\\
    &=  2d^2+1.
    \end{split}
\end{equation}
Collecting \(T\) projected Hessians requires: $Q_{\mathrm{curv}}(T,d) = T(2d^2+1)$ vector-output oracle queries. For \(d=64\), one complete Hessian requires \(8{,}193\) queries. Thus, for $T=8,\;16,\;32$ requires $ 65{,}544, \;131{,}088, \;262{,}176$ structural queries respectively.
\paragraph{Attack optimization.}
As we observed above, the oracle returns the complete $k$-dimensional vector $g_{\theta}(x)$ and therefore  $c^\top g_{\theta}(x)$ can be computed locally and requires no additional oracle query. But this means that the same finite-difference stencil at one probe point gives the Hessian of every output coordinate simultaneously. At a fixed probe point $x_p$,
\begin{equation}
\nabla^2 g_r(x_p) =
\sum_j W_{2,rj}\phi''(w_j^\top x_p+b_j)w_jw_j^\top ,
\qquad r=1,\ldots,k.
\end{equation}

After the $(2d^2+1)$ vector-valued oracle calls have been made, arbitrary projected Hessians
\begin{equation} 
\nabla^2(c^\top g)(x_p) =
\sum_r c_r\nabla^2g_r(x_p)
\end{equation} 
can be produced offline at \textbf{no additional query cost}. Thus, the optimized attack query cost for $P$ distinct probe locations is 
\begin{equation} P(2d^2+1),
\end{equation}
while many output projections can be obtained at each $x_p$. This represents a substantial increase of the power of the attack over the estimates in the previous paragraph. All curvature-extraction experiments therefore use this optimized shared-stencil attack, with structural queries counted as \(P(2d^2+1)\). 

Note that $P=1$ is exceptionally query-efficient when the probe activates sufficient curvature across units, while additional probes provide a robustness mechanism when some units have weak curvature or the output dimension/rank is insufficient. For example, for \(d=64\), the single-probe operating point requires \(8{,}193\) queries for the evaluated \(P=1\) configuration. Additional probes may be used to obtain more complete and robust recovery when some units have weak curvature or the output dimension/rank is insufficient. At a fixed probe, unit $j$ contributes through
\begin{equation} 
\beta_j =\phi''(w_j^\top x_p+b_j)\|w_j\|^2.
\end{equation} 
If $\beta_j=0$, that direction is invisible in every output projection at that probe. If it is very small, it can be poorly conditioned. Similarly, projection diversity at one probe is limited by  $\operatorname{rank}(W_2 D)\le k$, where $D(x_p)=\operatorname{diag}(\beta_1,\ldots,\beta_m)$. Thus, arbitrarily increasing $T$ beyond the available output rank cannot create new independent information.

\paragraph{Privileged diagnostic modes.}
For diagnostic comparison, we also construct Hessians from the analytic expression in Lemma~\ref{lem:hessian-mixture} and through automatic differentiation. Both modes require access to target-model parameters or derivatives and therefore do not constitute black-box extraction. The finite-difference mode is the principal oracle-based setting studied in this work.

\subsection{Shared Rank-One Dictionary Recovery}
\label{sec:dictionary-recovery}
Let $\widehat{\mathcal H} = \{\widehat H_t\}_{t=1}^{T}$ denote the projected Hessians estimated from oracle responses. $\mathcal{A}$
seeks a shared collection of normalized first-layer directions. For optimization, these directions are stacked as the columns: $U = [u_1,\ldots,u_m] \in \mathbb{R}^{d\times m}$, such that $\|u_j\|_2=1$.

\paragraph{Hessian normalization.}
The magnitudes of the projected Hessians vary with the probe points and output projections. We therefore compute the global root-mean-square scale:
\begin{equation}
\label{eq:hessian-rms}
    s_H = \left( \frac{1}{Td^2} \sum_{t=1}^{T} \|\widehat H_t\|_F^2 \right)^{1/2}
\end{equation}
and normalize each measurement as: $\widetilde H_t = \frac{\widehat H_t}{s_H}$. This rescaling preserves the relative magnitudes of the Hessians and
does not alter their shared direction atoms. We also introduce measurement-specific coefficients: $A = [a_{t,j}] \in \mathbb{R}^{T\times m}$, and seek a joint factorization of the form: $\widetilde H_t \approx \sum_{j=1}^{m} a_{t,j}u_ju_j^\top$. The coefficients absorb both the measurement-dependent curvature weights and the common normalization factor \(s_H\).

\paragraph{Dictionary objective.}
The normalized Hessians are jointly factorized by minimizing:
\begin{equation}
\label{eq:dictionary-objective}
    \begin{split}
        \mathcal{L}_{\mathrm{dict}}(U,A)={}& \frac{1}{Td^2} \sum_{t=1}^{T} \left\| \widetilde H_t -\sum_{j=1}^{m} a_{t,j}u_ju_j^\top \right\|_F^2 + \lambda_A \frac{\|A\|_F^2}{Tm},
    \end{split}
\end{equation}
subject to $\|u_j\|_2=1$, for $j=1,\ldots,m$. The coefficient penalty controls the magnitude of \(A\), while the unit-norm constraints remove the continuous scaling ambiguity between each direction and its measurement-specific coefficients.

\paragraph{Multi-restart joint optimization.}
The dictionary objective is nonconvex and may contain multiple local minima. We therefore optimize it from \(R\) independent random initializations. At each restart, the columns of \(U\) are sampled randomly and normalized to unit length. Conditional on the initial dictionary, each coefficient row \(a_{t,\cdot}\) is initialized by ridge regression of \(\operatorname{vec}(\widetilde H_t)\) on the vectorized rank-one atoms: $\left\{\operatorname{vec}(u_ju_j^\top) \right\}_{j=1}^{m}$. The variables \(U\) and \(A\) are then optimized jointly using Adam, with every column of \(U\) renormalized after each update. Restart selection uses only the unregularized Hessian reconstruction error
\begin{equation}
\label{eq:dictionary-reconstruction-loss}
    \mathcal{L}_{\mathrm{rec}}(U,A) = \frac{1}{Td^2} \sum_{t=1}^{T} \left\| \widetilde H_t - \sum_{j=1}^{m} a_{t,j}u_ju_j^\top \right\|_F^2.
\end{equation}
The restart attaining the smallest value of \(\mathcal{L}_{\mathrm{rec}}\) is retained. No target-model parameters, ground-truth directions, hidden activations, or structural evaluation scores are used for restart selection. The structural output of the attack is given by: $ \widehat U = \{\widehat u_1,\ldots,\widehat u_m\}$. Algorithm~\ref{alg:curvature-extraction} summarizes the complete structural extraction procedure. The algorithm below assumes that the complete finite-difference stencil remains within the query domain \(\mathcal X\). In the primary representation-space setting, \(\mathcal X=\mathbb{R}^{d}\); for a restricted domain, invalid probe points must be rejected or adjusted. The recovered set \(\widehat U\) is obtained solely from oracle-derived curvature measurements, the known architectural dimensions, and public optimization settings. It requires neither natural-data completion queries nor access to target-model parameters, hidden activations, or ground-truth directions. Thus, the structural extraction stage ends with a direct black-box estimate of the first-layer direction dictionary, independently of whether a functional surrogate is subsequently constructed.

\begin{algorithm}[tb]
\caption{Curvature Extraction of a Smooth Transformer FFN}
\label{alg:curvature-extraction}
\begin{algorithmic}[1]
\Require Chosen-input raw-output oracle
\(\mathcal{O}_{g}\);
query domain \(\mathcal X\);
dimensions \(d,k,m\);
number of probe points \(P\);
number of projections  \(T_p\) at probe $p$;
number of Hessians $T=\sum_{p=1}^P T_p$ ;
probe scale \(\sigma_x\);
finite-difference step \(h\);
restarts \(R\);
optimization steps \(S\)
\Ensure Recovered normalized directions
\(\widehat U=\{\widehat u_j\}_{j=1}^{m}\)

\State Initialize
\(\widehat{\mathcal H}\leftarrow\emptyset\)

\For{\(p=1,\ldots,P\)}
    \State Sample $x_p\sim\mathcal{N}(0,\sigma_x^2I_d)$  such that the complete finite-difference stencil lies in \(\mathcal X\)

    \State Query each point in the complete vector-valued finite-difference stencil around $x_p$ once, for a total of $2d^2+1$ oracle calls
    \State From those vector responses, estimate $\{\widehat{\nabla^2 g_r(x_p)\}_{r=1}^k}$
    \For{$t=1,\ldots, T_p$}
    \State Sample $c_{p,t}\sim \mathcal{N}(0,I_k)$
    \State Construct offline $\widehat H_{p,t} \gets \sum_{r=1}^k c_{p,t,r}\widehat{\nabla^2 g_r(x_p)}$ 

    \State Append $\widehat H_{p,t}$ to $\widehat{\mathcal H}$
\EndFor
\EndFor
\State Reindex the collected $\{ \widehat H_{p,t}\}$ as $\{ \widehat H_{t}\}_{t=1}^T$, where $T=\sum_{p=1}^P T_p$ 
\State Compute \(s_H\) using Equation~\eqref{eq:hessian-rms}
\For{$t=1,\ldots, T$}
\State Set $\widetilde H_t \leftarrow \widehat H_t/s_H$
\EndFor
\For{\(r=1,\ldots,R\)}
    \State Randomly initialize and column-normalize
    \(U^{(r)}\in\mathbb{R}^{d\times m}\)

    \State Initialize each row of
    \(A^{(r)}\in\mathbb{R}^{T\times m}\)
    by ridge regression

    \State Jointly optimize
    \(U^{(r)}\) and \(A^{(r)}\) for \(S\) steps using
    \(\mathcal{L}_{\mathrm{dict}}\)

    \State Renormalize every column of \(U^{(r)}\) after each update

    \State Record $\mathcal{L}_{\mathrm{rec}} \bigl(U^{(r)},A^{(r)}\bigr)$
\EndFor

\State Select \begin{equation*}r^*\gets\argminA_{r=1,\ldots, R} \mathcal{L}_{\mathrm{rec}} \bigl(U^{(r)},A^{(r)}\bigr)\end{equation*} 
\State Set  $\widehat U\gets U^{(r^*)}$
\State \Return $\widehat U$

\end{algorithmic}
\end{algorithm}

\subsection{Surrogate Completion and Full-Model Replacement}
\label{sec:surrogate-completion}
Structural recovery identifies the normalized rows of the first linear map, but does not determine their magnitudes, orientations, biases, or the second-layer parameters. To evaluate whether the recovered directions support functional replacement, we construct a surrogate FFN whose first-layer directions are fixed to \(\widehat U\).

\paragraph{Fixed-direction parameterization.}
Let $ \widehat U = [\widehat u_1,\ldots,\widehat u_m] \in \mathbb{R}^{d\times m}$. Then we parameterize the first-layer matrix as: $\widehat W_1 = \operatorname{diag}(s_1,\ldots,s_m)\widehat U^\top$, where $s_j\in\mathbb{R}$. The signed scales absorb the orientation ambiguity left by curvature recovery: replacing \(\widehat u_j\) by \(-\widehat u_j\) can be compensated by replacing \(s_j\) by \(-s_j\). The underlying unoriented direction set therefore remains unchanged. The completed branch is: $ \widehat g(x) = \widehat W_2 \phi(\widehat W_1x+\widehat b_1) + \widehat b_2$. During completion, \(\widehat U\) remains fixed, while the adversary optimizes: $ s, \;\widehat b_1, \;\widehat W_2, \; \widehat b_2$.

\paragraph{Completion objective.}
Let $\left\{ x_n^{\mathrm{nat}} \right\}_{n=1}^{N_{\mathrm{fit}}}$ be naturally occurring FFN input representations, and let $y_n = \mathcal{O}_g \left( x_n^{\mathrm{nat}} \right)$ be the corresponding branch outputs. The completion parameters are fitted by minimizing
\begin{equation}
\label{eq:surrogate-fit}
    \mathcal{L}_{\mathrm{fit}} = \frac{1}{N_{\mathrm{fit}}} \sum_{n=1}^{N_{\mathrm{fit}}} \left\| \widehat g\left( x_n^{\mathrm{nat}} \right) - y_n \right\|_2^2.
\end{equation}
The optimization is performed using AdamW while keeping the recovered direction matrix fixed. Under the stated block-level oracle model, each branch output \(y_n\) requires one additional oracle query. The completion query budget is therefore: $Q_{\mathrm{fit}} = N_{\mathrm{fit}}$. The total query cost of structural extraction followed by optional completion is: \( Q_{\mathrm{total}} = P(2d^2+1) + N_{\mathrm{fit}}\). We report the structural and completion budgets separately because only the first term is required to recover the hidden direction set. Functional completion is a separate empirical fitting stage built on the recovered direction dictionary. The local identifiability result does not guarantee recovery of the signed row scales, biases, or second-layer parameters, nor does it guarantee global functional equivalence. The completion experiment instead tests whether the structurally recovered directions provide a sufficiently accurate basis for replacing the target FFN in its original classifier. Algorithm~\ref{alg:surrogate-completion} summarizes the completion stage. This represents a controlled downstream-utility experiment under additional access assumptions.

\begin{algorithm}[tb]
\caption{Fixed-Direction Surrogate Completion}
\label{alg:surrogate-completion}
\begin{algorithmic}[1]
\Require Recovered directions
\(\widehat U\);
natural branch inputs
\(\{x_n^{\mathrm{nat}}\}_{n=1}^{N_{\mathrm{fit}}}\);
oracle \(\mathcal{O}_g\);
activation \(\phi\)
\Ensure Completed surrogate branch \(\widehat g\)

\For{\(n=1,\ldots,N_{\mathrm{fit}}\)}
    \State $y_n \leftarrow \mathcal{O}_g \left( x_n^{\mathrm{nat}}
    \right)$
\EndFor

\State Parameterize $\widehat W_1 \gets \operatorname{diag}(s_1,\ldots,s_m) \widehat U^\top$, such that $s_j\in\mathbb{R}$

\State Initialize $ s,\; \widehat b_1, \; \widehat W_2, \; \widehat b_2$

\State Jointly minimize: $ \frac{1}{N_{\mathrm{fit}}} \sum_{n=1}^{N_{\mathrm{fit}}} \left\| \widehat W_2 \phi \left( \operatorname{diag}(s) \widehat U^\top x_n^{\mathrm{nat}} + \widehat b_1
\right) + \widehat b_2 - y_n \right\|_2^2$, using AdamW while keeping \(\widehat U\) fixed

\State Set $ \widehat g(x) \gets \widehat W_2 \phi \left( \operatorname{diag}(\widehat s) \widehat U^\top x + \widehat b_1 \right) + \widehat b_2$
\State \Return $ \widehat g(x)$
\end{algorithmic}
\end{algorithm}

\paragraph{Full-model replacement.}
To evaluate downstream functional fidelity, we insert \(\widehat g\) into a copy of the target classifier while leaving all remaining model parameters unchanged. Let \(M\) denote the original classifier and let \(\widehat M\) denote the classifier containing the completed surrogate branch. On held-out inputs, we evaluate target and surrogate accuracy, top-1 predictive agreement, KL divergence, and centered-logit mean-squared error. For privileged scientific evaluation, we additionally compare \(\widehat U\) with the true normalized rows of \(W_1\), compute CKA between target and surrogate FFN representations, and measure matched hidden-feature correlation. These quantities require access to target-model parameters or internal activations and are not available to the adversary during extraction.

The attack therefore yields two distinct outcomes. Its primary output is a black-box reconstruction of the hidden first-layer direction dictionary. Its optional completion stage asks whether those extracted directions are sufficient to support an accurate functional replacement of the target branch. The experiments that follow evaluate these structural and functional claims separately, ensuring that downstream surrogate performance does not substitute for direct evidence of parameter-structure recovery.

\section{Experiments}
\label{sec:experiments}
We evaluate the proposed extraction procedure on independently trained transformer FFNs with GELU and SiLU activations. The experiments address four questions: whether the first-layer directions can be recovered under finite-difference black-box access, whether recovery is reproducible across models and extraction runs, whether the recovered structure supports functional replacement, and whether simple output perturbations remain effective against an adaptive choice of finite-difference scale.

\paragraph{Target models and data.}
We train vision transformers on CIFAR-10~\cite{krizhevsky2009learning}. Each \(32\times32\) image is partitioned into non-overlapping \(4\times4\) patches, producing \(64\) patch tokens and one classification token. Each model contains four transformer blocks, four attention heads, embedding dimension \(d=64\), and FFN hidden width \(m=128\). The FFN activation is either GELU or SiLU. Unless stated otherwise, extraction targets the FFN in the final transformer block. We use \(45{,}000\) training examples and \(5{,}000\) validation examples from the official training set. The \(10{,}000\)-example test set is used only after checkpoint selection. For each activation, we train three independently initialized target models using seeds \(\{0,1,2\}\). Training uses AdamW~\cite{loshchilov2017decoupled} with \(\mathrm{cosine}\) learning-rate decay for \(150\) epochs, batch size \(128\), initial learning rate \(3\times10^{-4}\), weight decay \(5\times10^{-2}\), and dropout \(0.05\). Table~\ref{tab:target-training} reports the validation-selected checkpoints. All experiments were conducted on a workstation equipped with an Intel Core Ultra 9 CPU, $64$\,GB of RAM, and an NVIDIA GeForce RTX 4070 GPU with \(8\)\,GB of GDDR6 memory. The implementation used PyTorch 2.5.1~\cite{paszke2019pytorch} with CUDA 12.4.

\begin{table}[t]
    \centering
    \caption{Validation-selected CIFAR-10 target models. The test set is not
    used for checkpoint selection.}
    \label{tab:target-training}
    \begin{tabular}{lcccc}
        \toprule
        Activation & Seed & Best epoch & Val. acc. (\%) & Test acc. (\%) \\
        \midrule
        GELU & 0 & 143 & 77.78 & 79.18 \\
        GELU & 1 & 129 & 79.06 & 78.83 \\
        GELU & 2 & 120 & 79.32 & 79.54 \\
        \midrule
        SiLU & 0 & 149 & 77.46 & 78.16 \\
        SiLU & 1 & 129 & 77.34 & 77.71 \\
        SiLU & 2 & 138 & 78.10 & 78.68 \\
        \bottomrule
    \end{tabular}
\end{table}

\paragraph{Extraction protocol.}
All extraction experiments query the post-LayerNorm FFN branch in evaluation mode, with dropout disabled. Projected Hessians are estimated using the centered finite-difference procedure and the optimized attack of Section~\ref{sec:projected-hessian-collection}; for \(d=64\), each distinct probe stencil requires \(8{,}193\) vector-output oracle queries. Unless stated otherwise, dictionary recovery uses three random restarts with joint Adam optimization, and the selected restart is determined solely by Hessian reconstruction loss. Ground-truth directions are used only for evaluation.

\paragraph{Default operating point.}
We evaluate $T\in\{4,8,16,32, 64\}$ projected Hessians at a single probe location. Recovery increases sharply between \(T=8\) and \(T=16\) and changes only marginally at \(T=32\) and \(T=64\). We therefore use $T=16$. Under the optimized vector-output stencil (cf., \textbf{Attack optimization } in section~\ref{sec:projected-hessian-collection}), all \(T\) projected Hessians are constructed offline from the same oracle responses. Thus, for \(d=64\) and \(P=1\), the default structural-extraction budget is: $Q_{\mathrm{struct}} = P(2d^2+1) = 1\cdot(2\cdot64^2+1) = 8{,}193$. We similarly vary the finite-difference step \(h\). Recovery is poor for \(h\leq3\times10^{-3}\), reaches the stable high-recovery regime at \(h=10^{-2}\), and shows no consistent improvement at larger tested values. We therefore use $h=10^{-2}$ for clean-oracle experiments unless stated otherwise. For optional functional completion, we use $N_{\mathrm{fit}} = Q_{\mathrm{fit}} = 12{,}000$ natural branch input--output pairs. Agreement improves rapidly at lower budgets and changes only marginally when the completion budget is increased to \(24{,}000\). The remaining FFN parameters are optimized for \(1{,}500\) AdamW steps over this fixed completion set, requiring no additional oracle queries. The combined budget is therefore:
\begin{equation}
    \begin{split}
        Q_{\mathrm{total}} &= Q_{\mathrm{struct}}+Q_{\mathrm{fit}}\\
        &= 8{,}193 + 12{,}000\\
        &= 20{,}193.
    \end{split}
\end{equation}
Structural and completion queries are reported separately because completion is not required to recover the first-layer direction dictionary.

\begin{figure}[t]
     \centering
     \begin{subfigure}[b]{0.48\textwidth}
         \centering
         \includegraphics[width=\textwidth]
         {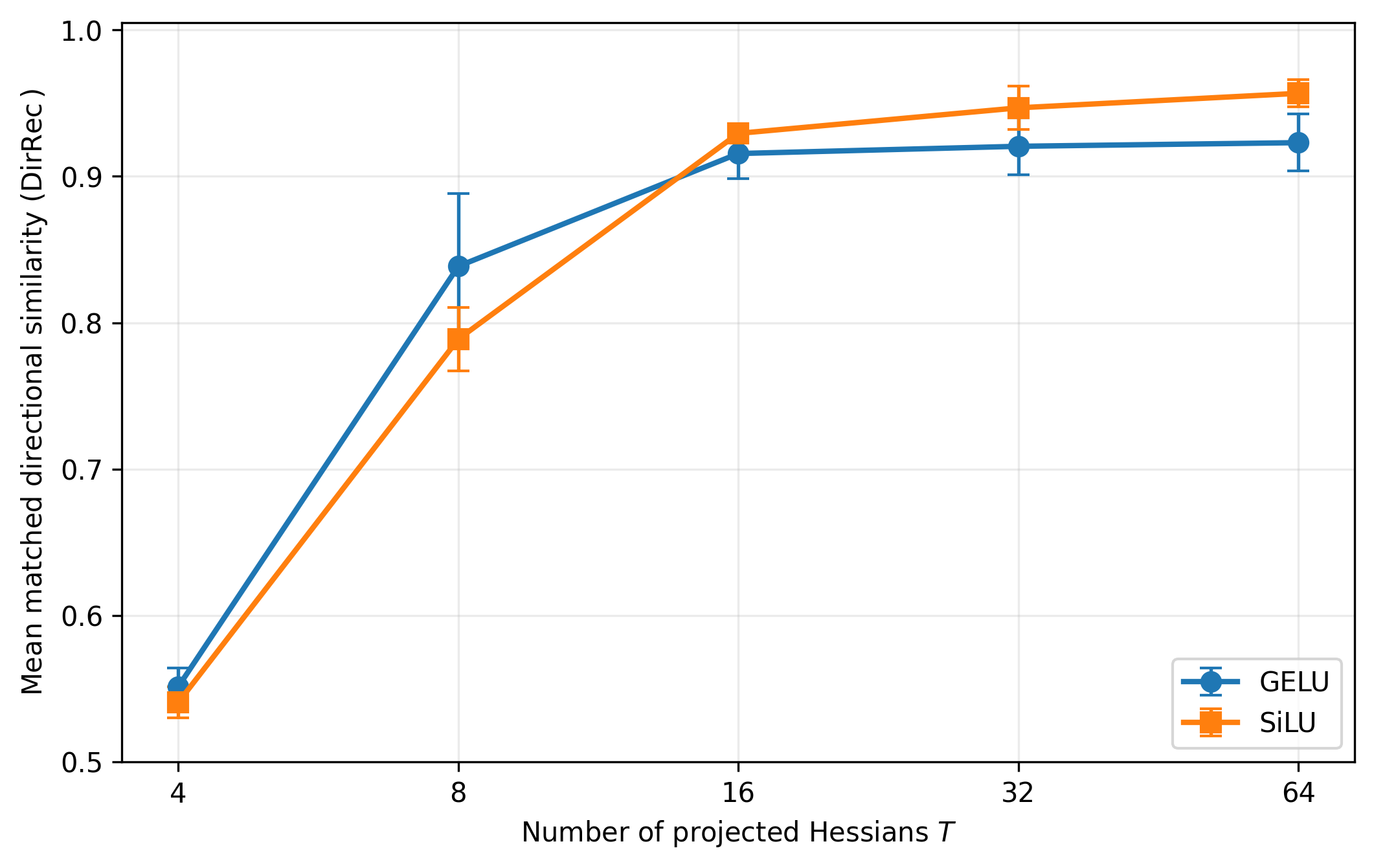}
         \caption{Recovery improves sharply through \(T=16\) and changes only marginally for \(T\geq32\).}
         \label{fig:exp-query-and-projected-hessian-ablation}
     \end{subfigure}
     \hfill
     \begin{subfigure}[b]{0.48\textwidth}
         \centering
         \includegraphics[width=\textwidth]
         {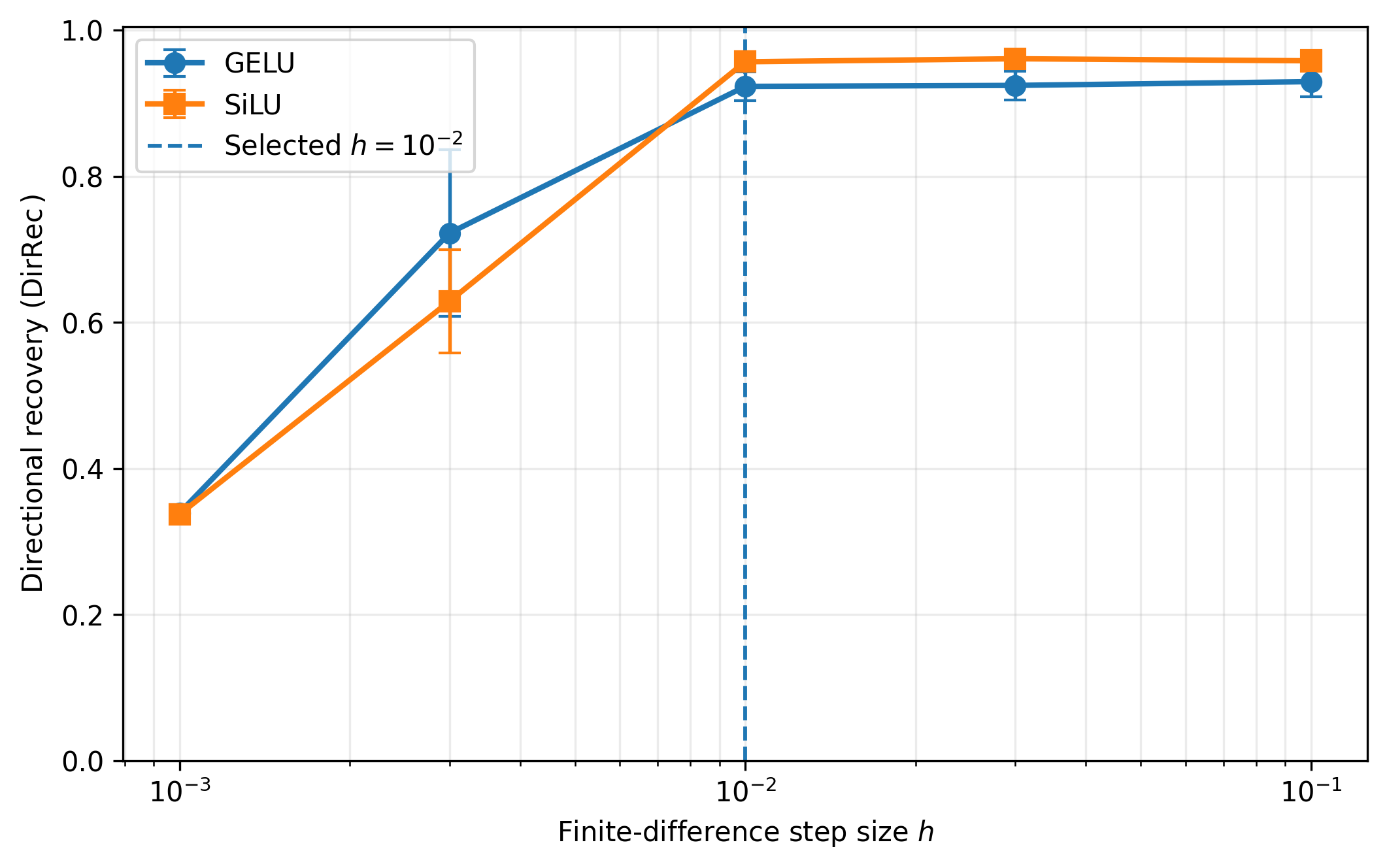}
         \caption{\(h=10^{-2}\) is the smallest tested step in the stable high-recovery regime.}
         \label{fig:exp-finite-difference-ablation}
     \end{subfigure}
     \caption{Selection of the optimized structural-extraction operating point. Increasing the number of projected Hessians improves recovery, with most of the gain attained by \(T=16\), while finite-difference recovery stabilizes at \(h=10^{-2}\). All experiments use one shared probe stencil and therefore require \(8{,}193\) structural vector-output oracle queries. Points and error bars report the mean and sample standard deviation across \(3\) independent attack seeds.}
     \label{fig:exp-query-and-finitedifference-ablation}
\end{figure}

\begin{figure}[t]
     \centering
     \begin{subfigure}[b]{0.48\textwidth}
         \centering
         \includegraphics[width=\textwidth]
         {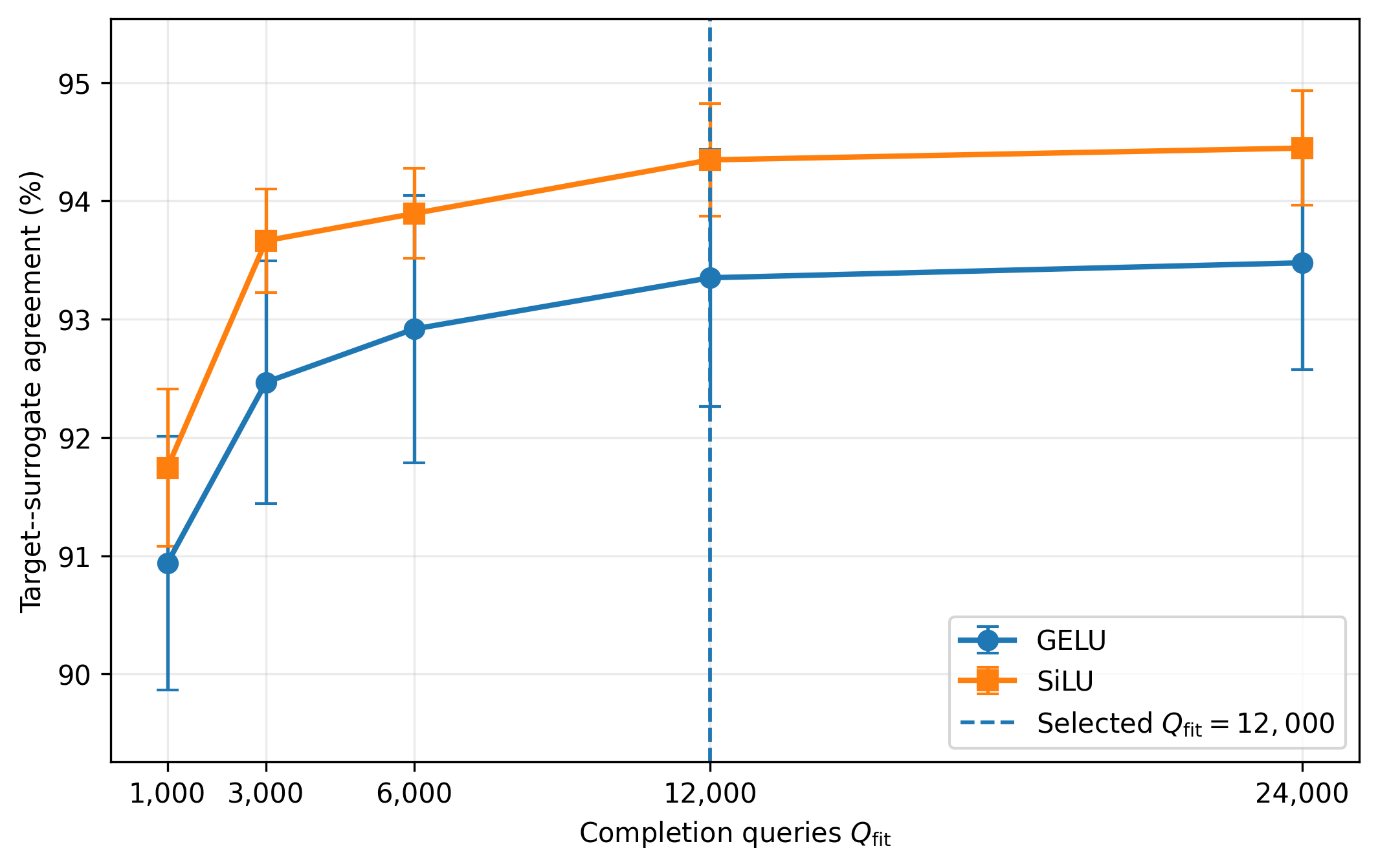}
         \caption{Agreement rises rapidly and approaches saturation by \(Q_{\mathrm{fit}}=12{,}000\).}
         \label{fig:exp-completion-query-ablation-surrogate}
     \end{subfigure}
     \hfill
     \begin{subfigure}[b]{0.48\textwidth}
         \centering
         \includegraphics[width=\textwidth]
         {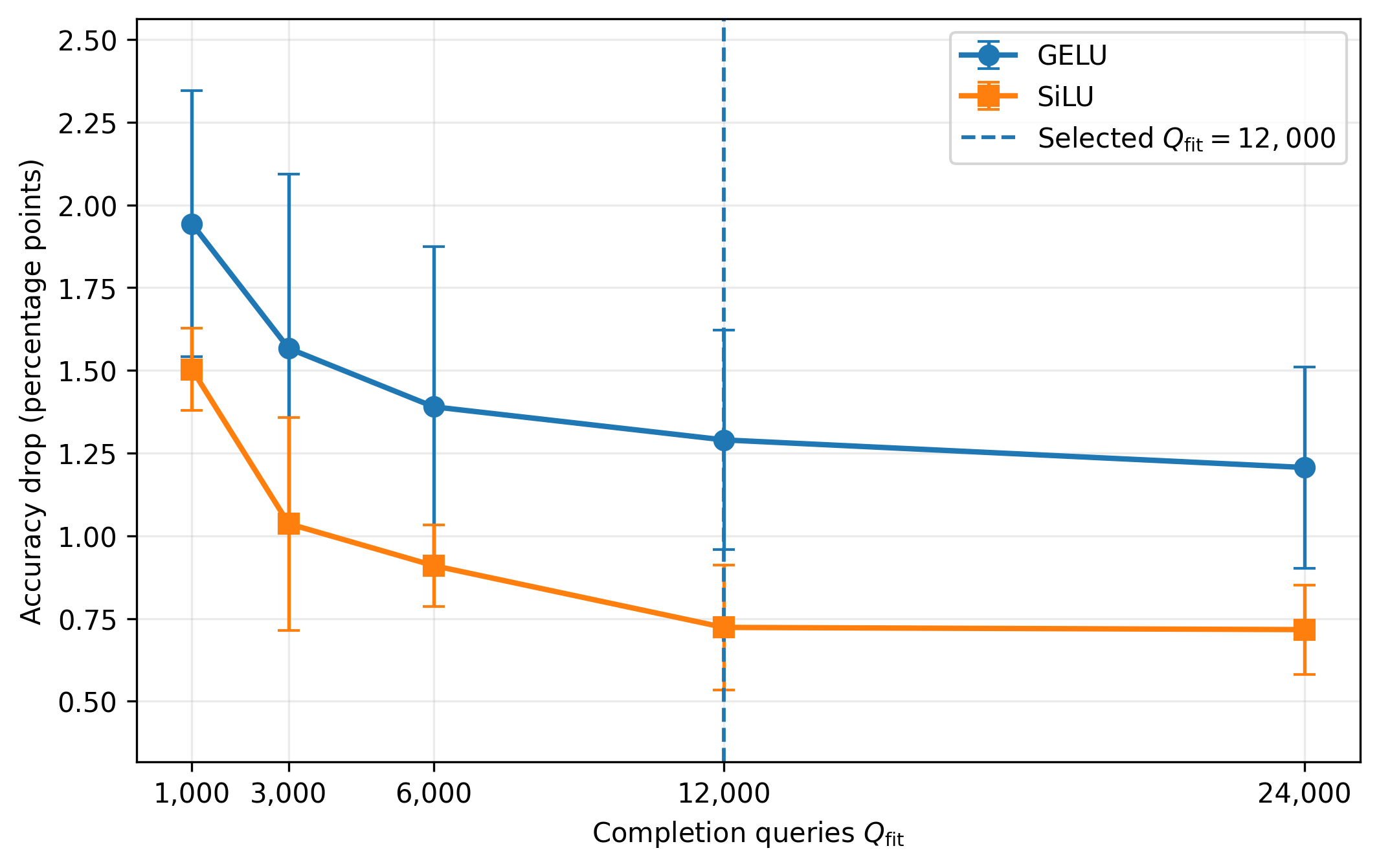}
         \caption{Completion sharply reduces the surrogate accuracy gap, with near-best preservation at \(Q_{\mathrm{fit}}=12{,}000\).}
         \label{fig:exp-completion-query-ablation-accuracy}
     \end{subfigure}
     \caption{Completion-query ablation for functional replacement. Increasing \(Q_{\mathrm{fit}}\) substantially improves target--surrogate agreement and accuracy preservation, with performance effectively saturating by \(12{,}000\) queries. Points show means over three independent extraction seeds; error bars denote one sample standard deviation.}
     \label{fig:exp-completion-query-ablation}
\end{figure}

\paragraph{Structural recovery metrics.}
Structural recovery is evaluated using the permutation- and sign-invariant metrics from Definition~\ref{def:direction-recovery}. Our primary metric is \(\operatorname{DirRec}\), the mean absolute cosine alignment after optimal one-to-one matching between the recovered and true first-layer directions. We also report \(\operatorname{Hit}_{0.90}\), the fraction of matched directions whose absolute cosine similarity is at least \(0.90\), together with the matched median and minimum alignment. As stated earlier, the true first-layer directions are used only for scientific evaluation and are not available to the adversary during extraction. We additionally report the final Hessian reconstruction loss as an optimization diagnostic. A low reconstruction loss alone does not establish structural recovery because distinct or poorly conditioned factorizations may reproduce the observed Hessians without accurately identifying the underlying direction atoms.

\paragraph{Functional replacement metrics.}
Functional replacement is evaluated only after structural recovery has been established. During completion, the recovered directions remain fixed while signed first-layer scales, first-layer biases, second-layer weights, and second-layer biases are fitted from additional branch input--output pairs.

The completed FFN is inserted into the original classifier while all remaining parameters are held fixed. We report target and surrogate test accuracy, accuracy drop, top-1 prediction agreement, KL divergence, and centered-logit mean-squared error. For privileged scientific evaluation, we also report FFN representation CKA and matched hidden-feature correlation. These metrics measure the downstream usefulness of the recovered structure, but they do not replace direct structural evaluation. Predictive agreement may be achieved without identifying the target FFN's internal directions, whereas high \(\operatorname{DirRec}\) provides direct evidence that the first-layer geometry is exposed through the black-box curvature channel.

\subsection{Structural Recovery}
\label{sec:structural-recovery}
We first evaluate whether black-box curvature measurements reveal the normalized input directions of the first linear map \(W_1\) in the targeted FFN branch. This is the primary extraction criterion: unlike behavioral imitation, high directional recovery provides direct evidence that the target model's internal first-layer geometry is exposed. We examine recovery across independently trained target models, independent realizations of the extraction procedure, offline projection budgets, and transformer depth. Privileged exact-Hessian and initialization ablations then separate finite-difference measurement error from limitations of the nonconvex dictionary optimization.

\paragraph{Recovery across model and extraction randomness.}
Table~\ref{tab:independent-models} reports structural recovery and downstream functional replacement across three independently trained target models for each activation. The extraction seed is held fixed across target models so that the reported variation primarily reflects model-training randomness. Each extraction nevertheless uses three independent dictionary restarts, with the restart attaining the lowest observed Hessian reconstruction loss retained.

\begin{table}[t]
    \centering
    \caption{Structural recovery and functional replacement across independently trained target models. Values are mean \(\pm\) sample standard deviation over three training seeds. Each extraction uses \(T=16\), corresponding to 8,193 structural oracle queries, followed by 12,000 completion queries.}
    \label{tab:independent-models}
    \begin{tabular}{lcc}
        \toprule
        Metric & GELU & SiLU \\
        \midrule
        \(\operatorname{DirRec}\) & \(0.9187\pm0.0138\) & \(0.9402\pm0.0119\) \\
        \(\operatorname{Hit}_{0.90}\) & \(0.8672\pm0.0234\) & \(0.9115\pm0.0325\) \\
        Target accuracy (\%) & \(79.18\pm0.36\) & \(78.18\pm0.49\) \\
        Surrogate accuracy (\%) & \(77.94\pm0.47\) & \(77.61\pm0.49\) \\
        Accuracy drop (pp) & \(1.24\pm0.23\) & \(0.57\pm0.09\) \\
        Top-1 agreement (\%) & \(92.81\pm0.22\) & \(94.97\pm0.36\) \\
        KL divergence & \(0.0432\pm0.0030\) & \(0.0209\pm0.0031\) \\
        Centered-logit MSE & \(2.6383\pm0.1761\) & \(1.2743\pm0.1840\) \\
        FFN CKA & \(0.7979\pm0.0340\) & \(0.9034\pm0.0129\) \\
        Matched feature correlation & \(0.7128\pm0.0095\) & \(0.8107\pm0.0015\) \\
        \bottomrule
    \end{tabular}
\end{table}

Curvature measurements consistently expose the first-layer input directions of independently trained smooth FFNs. For GELU, \(86.7\%\) of the hidden directions are recovered with absolute cosine similarity of at least \(0.90\); for SiLU, the corresponding fraction is \(91.1\%\). Recovery remains consistent, with modest variation across independently trained checkpoints: GELU \(\operatorname{DirRec}\) ranges from \(0.9042\) to \(0.9317\), while SiLU ranges from \(0.9270\) to \(0.9503\). The observed leakage is therefore not specific to one favorable training seed or checkpoint.

We separately isolate randomness in the extraction procedure by fixing one target model per activation and varying the extraction seed. This changes the probe inputs, output projections, and dictionary initialization while leaving the target model unchanged. Each extraction seed again uses three internal dictionary restarts.

\begin{table}[t]
    \centering
    \caption{Stability across extraction seeds for one fixed target model per activation. Values are mean \(\pm\) sample standard deviation over three extraction seeds.}
    \label{tab:attack-seed-stability}
    \begin{tabular}{lcc}
        \toprule
        Metric & GELU & SiLU \\
        \midrule
        \(\operatorname{DirRec}\) & \(0.9102\pm0.0214\) & \(0.9374\pm0.0161\) \\
        \(\operatorname{Hit}_{0.90}\) & \(0.8516\pm0.0413\) & \(0.9036\pm0.0316\) \\
        Surrogate accuracy (\%) & \(77.94\pm0.22\) & \(77.51\pm0.13\) \\
        Top-1 agreement (\%) & \(93.27\pm0.38\) & \(94.90\pm0.16\) \\
        KL divergence & \(0.0376\pm0.0071\) & \(0.0195\pm0.0027\) \\
        FFN CKA & \(0.8284\pm0.0049\) & \(0.9024\pm0.0117\) \\
        \bottomrule
    \end{tabular}
\end{table}

The narrow standard deviations show that successful recovery does not depend on one favorable probe set, output projection, or optimizer trajectory. Together, Tables~\ref{tab:independent-models} and~\ref{tab:attack-seed-stability} isolate the two principal sources of experimental variability: recovery persists across independently trained target models and across independent realizations of the adversarial procedure.

\paragraph{Effect of the number of projected Hessians.}
We vary the number of projected Hessians over \(T\in\{4,8,16,32,64\}\). Under the optimized attack, all \(T\) projected Hessians are constructed offline from a single shared finite-difference stencil. Consequently, every configuration requires the same \(2d^2+1=8{,}193\) vector-output oracle calls. The results exhibit a clear transition from insufficient measurements to stable structural recovery.

\begin{table}[t]
    \centering
    \caption{Structural recovery as a function of the number of projected Hessians. Each configuration uses a fixed extraction seed and the same \(8{,}193\) vector-output oracle calls.}
    \label{tab:query-budget}
    \begin{tabular}{lrrr}
        \toprule
        Activation & \(T\) &
        \(\operatorname{DirRec}\) &
        \(\operatorname{Hit}_{0.90}\) \\
        \midrule
        GELU & \(4\)  & \(0.5292\) & \(0.0000\) \\
        GELU & \(8\)  & \(0.8431\) & \(0.6641\) \\
        GELU & \(16\) & \(0.9201\) & \(0.8672\) \\
        GELU & \(32\) & \(0.9333\) & \(0.8828\) \\
        \midrule
        SiLU & \(4\)  & \(0.5226\) & \(0.0000\) \\
        SiLU & \(8\)  & \(0.7535\) & \(0.3594\) \\
        SiLU & \(16\) & \(0.9432\) & \(0.9219\) \\
        SiLU & \(32\) & \(0.9634\) & \(0.9453\) \\
        \bottomrule
    \end{tabular}
\end{table}

At \(T=4\), neither activation recovers a single direction above the \(0.90\) threshold. This is consistent with the relaxed necessary dimension count, which requires \(T\geq5\) for \(d=64\) and \(m=128\). At \(T=8\), the measurements contain substantial structural information, but recovery remains incomplete: for the fixed extraction seed, GELU attains \(\operatorname{DirRec}=0.8431\), while SiLU attains \(0.7535\). Recovery rises sharply at \(T=16\), whereas doubling the number of projected Hessians to \(T=32\) produces only modest additional gains. The low-measurement regime also distinguishes observation fitting from parameter identification. At \(T=4\), the dictionary optimizer can attain low Hessian reconstruction loss while recovering incorrect directions. Reconstruction loss is therefore useful for selecting among restarts once the measurements are sufficiently informative, but it is not itself a certificate of structural recovery.

\paragraph{Recovery across transformer depth.}
To test whether leakage is specific to the final transformer block, we apply the \(T=16\) extraction procedure to every FFN of one trained GELU target. Table~\ref{tab:block-depth} reports the resulting structural recovery. All four FFNs expose at least \(81.3\%\) of their hidden directions above \(0.90\) alignment, showing that the vulnerability is not confined to the final block. Block~2 achieves the highest recovery, with \(\operatorname{DirRec}=0.9391\) and \(\operatorname{Hit}_{0.90}=0.9063\).

\begin{table}[t]
    \centering
    \caption{Finite-difference structural recovery across all transformer blocks of one trained GELU target model.}
    \label{tab:block-depth}
    \begin{tabular}{rccc}
        \toprule
        Block &
        \(\operatorname{DirRec}\) &
        \(\operatorname{Hit}_{0.90}\) &
        Recon. loss \\
        \midrule
        \(0\) & \(0.8894\) & \(0.8125\) & \(0.00625\) \\
        \(1\) & \(0.9308\) & \(0.8672\) & \(0.00555\) \\
        \(2\) & \(0.9391\) & \(0.9063\) & \(0.00440\) \\
        \(3\) & \(0.9201\) & \(0.8672\) & \(0.00874\) \\
        \bottomrule
    \end{tabular}
\end{table}

\paragraph{Exact-curvature diagnostic and dictionary optimization.}
Exact Hessians serve only as a privileged diagnostic for separating factorization error from finite-difference error. At \(T=16\), GELU achieves \((\operatorname{DirRec},\operatorname{Hit}_{0.90})=(0.9341,0.8672)\) with exact Hessians and \((0.9201,0.8672)\) with finite differences. Thus, finite-difference recovery closely approaches privileged exact-curvature recovery, with identical fractions of directions exceeding \(0.90\) alignment. Table~\ref{tab:init-ablation} evaluates the initialization strategy. Because \(T\) centered Hessians span at most \(T-1=15\) dimensions while the FFN has \(m=128\) directions, sequential pursuit alone cannot initialize the full dictionary reliably. Joint Adam refinement substantially improves the pursuit solution, raising \(\operatorname{DirRec}\) from \(0.5714\) to \(0.9337\). Pursuit-initialized Adam and random-restart Adam attain comparable recovery, indicating that joint optimization is more important than the initialization strategy in this setting. For a uniform pipeline that does not require a separate pursuit stage, we use random initialization, joint Adam optimization, and restart selection by observed Hessian reconstruction loss throughout.

\begin{table}[t]
    \centering
    \caption{Dictionary-initialization ablation using exact Hessians for GELU at \(T=16\).}
    \label{tab:init-ablation}
    \begin{tabular}{lccc}
        \toprule
        Initialization &
        \(\operatorname{DirRec}\) & \(\operatorname{Hit}_{0.90}\) & Recon. loss \\
        \midrule
        Pursuit only & \(0.5714\) & \(0.1016\) & -- \\
        Pursuit + Adam & \(0.9337\) & \(0.8906\) & \(0.00212\) \\
        Random restarts + Adam & \(0.9341\) & \(0.8672\) & \(0.00222\) \\
        \bottomrule
    \end{tabular}
\end{table}

% Sequential pursuit fails to resolve the full overcomplete dictionary, and subsequent Adam refinement only partially repairs the initialization. Random multi-restart joint optimization instead recovers \(94.5\%\) of the hidden directions above \(0.90\) alignment. All primary experiments therefore use random initialization, joint Adam optimization, and restart selection based solely on the observed Hessian reconstruction objective.

Taken together, these results demonstrate structural recovery across both activations, multiple extraction seeds, every transformer block, and black-box finite-difference measurements. At \(T=16\), the attack consistently recovers most first-layer directions in trained GELU and SiLU FFNs. The sharp transition between \(T=8\) and \(T=16\) further shows that recovery arises from sufficiently diverse shared-curvature measurements rather than merely fitting the observed Hessians.

\subsection{Empirical Verification of Local Identifiability}
\label{sec:empirical-identifiability}

The structural recovery results above show a pronounced transition between \(T=8\) and \(T=16\). We now examine whether this transition is reflected in the local geometry of the Hessian factorization itself. Following Section~\ref{sec:local-identifiability}, we evaluate the restricted Jacobian \(J_{\mathrm{res}}(U,\Gamma)\) at each trained FFN instance and measure $\kappa = \sigma_{\min}\!\left(J_{\mathrm{res}}(U,\Gamma)\right)$,
together with the scale-normalized quantity $\frac{\kappa}{\sigma_{\max}(J_{\mathrm{res}})}$. A smallest singular value separated from zero provides numerical support for the injectivity condition in Theorem~\ref{thm:local-identifiability}, whereas its magnitude quantifies how well conditioned the local inverse problem is.

We evaluate all three independently trained checkpoints for both GELU and SiLU. For each checkpoint, we fix one probe $x_p$ and use a common nested sequence of output projections $c_t$ for $T\in \{4,8,16,32\}$. 

The coefficients \(\Gamma\) are constructed from the corresponding trained FFN parameters and probe--projection pairs, while \(\dot\Gamma\) remains unrestricted in the Jacobian test, matching the relaxed factorization analyzed in Theorem~\ref{thm:local-identifiability}. Table~\ref{tab:jacobian-identifiability} summarizes the results. For \(T=4\), the restricted Jacobian cannot be injective: its parameter dimension is \(8576\), while the Hessian observations provide only \(8320\) independent coordinates. This agrees exactly with the necessary dimension count in Corollary~\ref{cor:dimension-count}. At \(T=8\), however, all six trained FFNs already have numerically full-column-rank restricted Jacobians. The key distinction is conditioning. For GELU, the mean normalized smallest singular value is only \(0.0114\) at \(T=8\), with a mean condition number of \(90.4\). At \(T=16\), these improve sharply to \(0.0395\) and \(26.3\), respectively. For SiLU, the same transition is even more pronounced: the normalized smallest singular value increases from \(0.0058\) to \(0.0227\), while the mean condition number falls from \(196.1\) to \(45.2\). Conditioning improves further at \(T=32\).

\begin{table}[t]
	\centering
	\caption{Restricted-Jacobian identifiability diagnostics across three independently trained checkpoints. The results are computed with one common $x_p$ and $T$ output projections. Values are mean \(\pm\) standard deviation over training seeds. Here \(\kappa=\sigma_{\min}(J_{\mathrm{res}})\), while \(\kappa/\sigma_{\max}\) provides a scale-normalized measure of local conditioning. \(\dagger\):   Injectivity is ruled out by the necessary dimension count in Corollary~\ref{cor:dimension-count}.}
	\label{tab:jacobian-identifiability}
	\begin{tabular}{lccccc}
		\toprule
		Activation & \(T\) & \(\kappa\) &
		\(\kappa/\sigma_{\max}\) &
		\(\operatorname{cond}(J_{\mathrm{res}})\) & Injective \\
		\midrule
		GELU & 4  & -- & -- & -- & \(\dagger\) \\
		GELU & 8  & \(0.0266\pm0.0055\) &
		\(0.0114\pm0.0022\) & \(90.4\pm18.6\) & 3/3 \\
		GELU & 16 & \(0.1044\pm0.0144\) &
		\(0.0395\pm0.0089\) & \(26.3\pm6.8\) & 3/3 \\
		GELU & 32 & \(0.2047\pm0.0227\) &
		\(0.0632\pm0.0090\) & \(16.0\pm2.1\) & 3/3 \\
		\midrule
		SiLU & 4  & -- & -- & -- & \(\dagger\)\\
		SiLU & 8  & \(0.0123\pm0.0046\) &
		\(0.0058\pm0.0022\) & \(196.1\pm96.0\) & 3/3 \\
		SiLU & 16 & \(0.0489\pm0.0089\) &
		\(0.0227\pm0.0043\) & \(45.2\pm8.3\) & 3/3 \\
		SiLU & 32 & \(0.1044\pm0.0093\) &
		\(0.0445\pm0.0018\) & \(22.5\pm0.9\) & 3/3 \\
		\bottomrule
	\end{tabular}
\end{table}

These results sharpen the interpretation of the recovery transition. At \(T=8\), local identifiability is already numerically supported, but the factorization remains substantially ill-conditioned. Increasing to \(T=16\) does not merely add more equations; it produces a markedly better-conditioned local inverse problem across every trained checkpoint. This improvement is consistent with the transition from partial recovery at \(T=8\) to reliable direction recovery at \(T=16\). The systematically weaker conditioning of SiLU is likewise consistent with its slightly lower structural recovery scores. We therefore view \(T=16\) as an operating point at which the curvature measurements are both locally informative and sufficiently well-conditioned for stable numerical extraction.

\subsection{Functional Replacement}
\label{sec:functional-replacement}
Structural recovery identifies the normalized input directions of \(W_1\), but does not by itself determine the complete FFN branch. We therefore evaluate whether the recovered directions provide a useful basis for high-fidelity functional replacement. During completion, the recovered directions remain fixed while signed first-layer scales, first-layer biases, second-layer weights, and second-layer biases are fitted from additional branch input--output queries. The completed FFN is then substituted into the full classifier, while all remaining target-model parameters are left unchanged.

The results in Table~\ref{tab:independent-models} show that the recovered structure supports accurate replacement across independently trained target models. GELU surrogates achieve \(77.94\%\pm0.47\%\) test accuracy, compared with \(79.18\%\pm0.36\%\) for the corresponding targets, yielding a modest accuracy drop of \(1.24\pm0.23\) percentage points. SiLU surrogates achieve \(77.61\%\pm0.49\%\), compared with \(78.18\%\pm0.49\%\) for the targets, yielding a drop of \(0.57\pm0.09\) percentage points. Top-1 prediction agreement remains \(92.81\%\pm0.22\%\) for GELU and \(94.97\%\pm0.36\%\) for SiLU.

The low KL divergence and centered-logit mean-squared error show that fidelity extends beyond the predicted class. FFN CKA and matched hidden-feature correlation further indicate that the completed branches reproduce substantial aspects of the target representation structure. Completion therefore does not merely recover aggregate test accuracy; it yields replacements whose predictions, logits, and hidden representations closely track those of the target models.

The completion-query ablation shows that this fidelity does not require the largest evaluated query budget. Target--surrogate agreement improves rapidly at lower budgets and largely stabilizes by: $Q_{\mathrm{fit}}=12{,}000$. Doubling the completion budget to \(24{,}000\) produces only marginal additional improvement.

Functional replacement complements, but does not replace, the structural claim. Behavioral agreement alone does not imply identification of a target model's internal parameters. Here, however, high predictive fidelity is obtained while the independently recovered \(W_1\) directions remain fixed. The completion results therefore demonstrate a practical consequence of structural leakage: the recovered direction dictionary provides a reliable foundation for constructing a high-fidelity substitute for the targeted FFN.

\subsection{Adaptive Robustness to Output Perturbations}
\label{sec:adaptive-robustness}
We evaluate whether simple perturbations of the branch-oracle responses prevent structural extraction. We consider two mechanisms: deterministic rounding of each returned output coordinate to four decimal places, equivalent to a quantization interval of \(\Delta=10^{-4}\), and independent additive Gaussian noise with standard deviation \(10^{-4}\). At the default clean-oracle step \(h=10^{-2}\), recovery collapses under both perturbations. Output rounding reduces \(\operatorname{DirRec}\) to \(0.3255\), while Gaussian noise reduces it to \(0.3295\); in both cases, \(\operatorname{Hit}_{0.90}=0\).

A fixed-step evaluation, however, does not represent an adaptive adversary. For bounded output perturbations, Equation~\eqref{eq:fd-error-tradeoff} gives the entrywise finite-difference tradeoff
\begin{equation}
\left|
\widehat H_{t,ij}-H_{t,ij}
\right|
=
O(h^2)
+
O\!\left(
\frac{\tau_{c_t}}{h^2}
\right),
\label{eq:adaptive-fd-tradeoff}
\end{equation}
where, the first term is the finite-difference truncation error and the second captures amplification of oracle-response errors. Increasing \(h\) suppresses the perturbation-amplification term, provided that the resulting truncation error remains controlled. Although Gaussian noise is not deterministically bounded, its contribution to each finite-difference estimate is likewise scaled by \(h^{-2}\), motivating the same adaptive choice of measurement scale.

\begin{figure}[t]
    \centering
    \includegraphics[width=0.98\columnwidth]
    {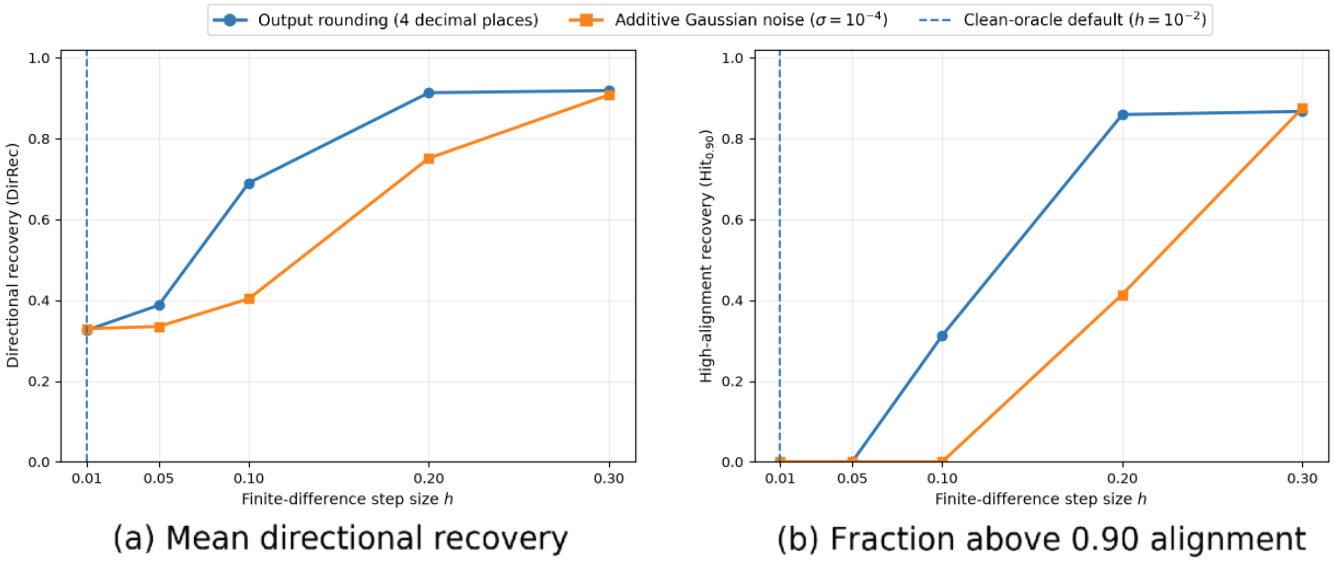}
    \caption{Adaptive structural recovery under perturbed oracle responses. At the default clean-oracle step \(h=10^{-2}\), rounding to four decimal places and additive Gaussian noise with standard deviation \(10^{-4}\) reduce recovery to the low-alignment regime. Increasing the finite-difference step suppresses perturbation amplification and restores strong recovery. Panel~(a) reports \(\operatorname{DirRec}\), while Panel~(b) reports \(\operatorname{Hit}_{0.90}\).}
    \label{fig:adaptive-output-perturbation}
\end{figure}

Figure~\ref{fig:adaptive-output-perturbation} shows that the apparent protection at \(h=10^{-2}\) is not robust to adaptation of the measurement scale. Under output rounding, \(\operatorname{DirRec}\) increases from \(0.3255\) at \(h=0.01\) to \(0.9188\) at \(h=0.30\). Under Gaussian noise, it increases from \(0.3295\) to \(0.9081\) over the same range. At \(h=0.30\), \(\operatorname{Hit}_{0.90}\) reaches \(0.8672\) under rounding and \(0.8750\) under Gaussian noise.

\begin{table}[H]
    \centering
    \caption{Adaptive recovery under four-decimal output rounding and
    Gaussian noise (\(\sigma=10^{-4}\)) for one GELU target
    (\(T=16\); fixed extraction seed).}
    \label{tab:adaptive-robustness}
    \begin{tabular}{llccc}
        \toprule
        Perturbation & \(h\) &
        \(\operatorname{DirRec}\) &
        \(\operatorname{Hit}_{0.90}\) &
        Recon. loss \\
        \midrule
        Rounding & \(0.01\) & \(0.3255\) & \(0.0000\) & \(0.1115\) \\
        & \(0.05\) & \(0.3875\) & \(0.0000\) & \(0.1912\) \\
        & \(0.10\) & \(0.6901\) & \(0.3125\) & \(0.1284\) \\
        & \(0.20\) & \(0.9135\) & \(0.8594\) & \(0.0160\) \\
        & \(0.30\) & \(0.9188\) & \(0.8672\) & \(0.0052\) \\
        \midrule
        Gaussian noise & \(0.01\) & \(0.3295\) & \(0.0000\) & \(0.2538\) \\
        & \(0.05\) & \(0.3348\) & \(0.0000\) & \(0.2523\) \\
        & \(0.10\) & \(0.4034\) & \(0.0000\) & \(0.2437\) \\
        & \(0.20\) & \(0.7513\) & \(0.4141\) & \(0.1022\) \\
        & \(0.30\) & \(0.9081\) & \(0.8750\) & \(0.0301\) \\
        \bottomrule
    \end{tabular}
\end{table}

These results show that both evaluated perturbation mechanisms can be substantially bypassed by adapting the finite-difference by selecting $h$ that minimizes the observed reconstruction loss. Note that evaluating $l$ different $h$'s costs $l\times 8193$ queries, if each candidate requires its own stencil. They do not establish that rounding or additive noise can never prevent extraction. Rather, they show that robustness measured at a single fixed value of \(h\) can substantially overestimate resistance to an adaptive adversary. Perturbation-based defenses should therefore be evaluated over adversarial choices of the finite-difference scale, together with the effect of the perturbation on legitimate model outputs.

\subsection{Summary of Empirical Findings}
\label{sec:experimental-summary}
The experiments establish that curvature-based extraction reveals genuine internal structure rather than merely reproducing target behavior. Using \(T=16\) projected Hessians constructed offline from a single \(8{,}193\)-query vector-output stencil, the attack attains mean \(\operatorname{DirRec}=0.9187\) for GELU and \(0.9402\) for SiLU, recovering \(86.7\%\) and \(91.1\%\) of hidden directions above \(0.90\) alignment. Recovery remains high across independently trained models, extraction seeds, and transformer blocks.

This recovery is nontrivial. At \(T=4\), no direction exceeds the \(0.90\) threshold despite low Hessian reconstruction error; \(T=8\) yields only partial recovery, while \(T=16\) enters a stable high-recovery regime. Moreover, joint Adam optimization substantially outperforms pursuit alone, while pursuit-initialized and random-restart Adam attain comparable recovery, showing that successful extraction requires both sufficiently diverse curvature measurements and joint factorization. The recovered directions are also operationally useful. After fixing them and fitting only the remaining FFN parameters, the resulting surrogates preserve more than \(92\%\) top-1 agreement and remain within \(1.24\) and \(0.57\) percentage points of GELU and SiLU target accuracy, respectively. Finally, output rounding and additive noise that appear protective at a fixed finite-difference step can be substantially bypassed by adapting the measurement scale. These results establish an end-to-end path from black-box second-order observations to hidden first-layer recovery and high-fidelity functional replacement. Under the stated oracle model, smooth FFN curvature forms an exploitable structural side channel.

\section{Conclusion}
\label{sec:conclusion}
We introduced a black-box curvature cryptanalysis of smooth transformer feed-forward networks under chosen-input raw-output access to a designated FFN branch. The key observation is that projected input Hessians share the same hidden rank-one factors induced by the first-layer weights. Joint factorization of these measurements therefore exposes the normalized first-layer direction dictionary, turning second-order responses into a structural extraction channel. Since the oracle returns the complete output vector, the \(T=16\) projected Hessians are constructed offline from a single \(2d^2+1=8{,}193\)-query stencil. On independently trained vision transformers with GELU and SiLU activations, the attack achieves mean \(\operatorname{DirRec}\) values of \(0.9187\) and \(0.9402\), recovering \(86.7\%\) and \(91.1\%\) of hidden directions above \(0.90\) absolute cosine alignment. Recovery remains high across training seeds, extraction runs, and transformer blocks. When the recovered directions are fixed and only the remaining FFN parameters are fitted, the resulting surrogates preserve more than \(92\%\) top-1 agreement while remaining within \(1.24\) and \(0.57\) percentage points of GELU and SiLU target accuracy, respectively.

These results show that the leakage is both structural and operational: black-box curvature measurements reveal hidden first-layer geometry that can be reused to construct a high-fidelity substitute. Finite-difference recovery is comparable to the privileged exact-Hessian diagnostic, and the evaluated rounding and noise defenses can be substantially weakened by adapting the finite-difference scale. Under the stated block-level oracle model, smooth transformer FFN curvature therefore constitutes a concrete structural side channel, extending model extraction beyond behavioral imitation to recovery of internal parameter geometry.

\paragraph{Scope beyond vision transformers.}
Although our experiments use vision transformers, the curvature mechanism acts on the FFN branch and is not inherently tied to visual inputs. Under the same architectural and oracle assumptions, the analysis extends to encoder-only, decoder-only, and encoder--decoder transformers whose feed-forward blocks have the smooth two-layer form: $g(x) = W_2\phi(W_1x+b_1)+b_2$. Many modern language models instead use gated FFNs, such as SwiGLU, whose second-order structure contains coupled terms involving both gate and value directions. Extending curvature cryptanalysis to these architectures requires a richer structured factorization and remains an important direction for future work.

%Bibliography
\bibliographystyle{unsrt}  
\bibliography{references} 

\end{document}